\documentclass{article}

\PassOptionsToPackage{numbers, compress}{natbib}

\usepackage[preprint]{neurips_2020}

\usepackage{microtype}
\usepackage{graphicx}
\usepackage{subfigure}
\usepackage{booktabs} 
\usepackage{tablefootnote}

\usepackage{hyperref}
\usepackage{url}
\usepackage{booktabs}       
\usepackage{amsfonts}       
\usepackage{nicefrac}       
\usepackage{microtype}      
\usepackage{makecell}

\usepackage{algorithm}
\usepackage{algpseudocode}
\usepackage{subfigure}
\usepackage{hyperref}
\usepackage{relsize}

\usepackage{amsmath,amssymb,amsfonts,amsthm,mathtools,bm}
\usepackage{enumitem}
\usepackage{multirow}
\usepackage{bm}
\usepackage{array}
\usepackage{color}
\usepackage{xcolor}
\usepackage{wrapfig}
\usepackage{multicol}
\usepackage{caption}
\usepackage{diagbox}
\usepackage{pifont}
\usepackage{pdfpages}

\usepackage{macro}

\title{A General Family of Stochastic Proximal Gradient Methods for Deep Learning}
\author{%
  Jihun Yun \\
  KAIST \\
  \texttt{arcprime@kaist.ac.kr} \\
  \And
  Aur\'elie C. Lozano \\
  IBM T.J. Watson Research Center \\
  \texttt{aclozano@us.ibm.com} \\
  \And
  Eunho Yang \\
  KAIST, AITRICS \\
  \texttt{eunhoy@kaist.ac.kr}
}

\begin{document}
\maketitle

\begin{abstract}
	We study the training of regularized neural networks where the regularizer can be non-smooth and non-convex. We propose
	a unified framework for stochastic proximal gradient descent, which we term \textsc{ProxGen}, that allows for arbitrary positive preconditioners and lower semi-continuous regularizers. Our framework encompasses standard stochastic proximal gradient methods \emph{without} preconditioners as special cases, which have been extensively studied in various settings.
	Not only that, we present two important update rules beyond the well-known standard methods as a byproduct of our approach: (i) the first closed-form proximal mappings of $\ell_q$ regularization $(0 \leq q \leq 1)$ for \emph{adaptive} stochastic gradient methods, and (ii) a revised version of \textsc{ProxQuant} \cite{bai2019proxquant} 
	that fixes a caveat of the original approach for quantization-specific regularizers.
	We analyze the convergence of \textsc{ProxGen} and show that the whole family of \textsc{ProxGen} enjoys the same convergence rate as stochastic proximal gradient descent without preconditioners.
	We also empirically show the superiority of proximal methods compared to subgradient-based approaches via extensive experiments. Interestingly, our results indicate that proximal methods with non-convex regularizers are more effective than those with convex regularizers.
\end{abstract}

\section{Introduction}

We study the regularized training of neural networks, which can be formulated as the following (stochastic) optimization problem
\begin{align}\label{eqn:problem}
\minimize\limits_{\theta \in \Omega}~ F(\theta) \coloneqq \mathbb{E}_{\xi \sim \mathbb{P}}\big[f(\theta; \xi)\big] + \mathcal{R}(\theta)
\end{align}
where $\theta \in \mathbb{R}^p$ represents the network parameter vector, $\xi$ is the random variable corresponding to mini-batch data samples, and $\mathcal{R}(\cdot)$ is a regularizer encouraging low-dimensional structural constraints on the parameter vector such as sparsity or low-rankness. 

For the \emph{unregularized} case, i.e., when $\mathcal{R}(\theta)=0$, stochastic gradient descent (SGD) has been a prevalent approach to solve the optimization problem \eqref{eqn:problem}. At each iteration, SGD evaluates the gradient only on a randomly chosen subset of training samples (mini-batch). Vanilla SGD employs a uniform learning rate for all coordinates, and several adaptive variants have been proposed, which scale the learning rate for each coordinate by its gradient history. A prime example of such approaches is \textsc{AdaGrad} \cite{duchi11}, which adjusts the learning rate by the sum of all the past squared gradients. However, the performance of \textsc{AdaGrad} degrades in non-convex dense settings as the learning rates vanish too rapidly. To resolve this issue, exponential moving average (EMA) approaches such as \textsc{RMSprop} \cite{tieleman2012lecture} and \textsc{Adam}~\cite{kingma15} have been proposed and become popular. These scale down the gradients by square roots of exponential moving averages of squared past gradients to essentially limit the scope of the adaptation to only a few recent gradients. In terms of theory, convergence analyses of these unregularized SGD, whether adaptive or not, have been well studied both for convex \cite{kingma15,j.2018on} and non-convex \cite{chen2019convergence,lei2019stochastic} loss $f$ cases.

The technique of regularization is ubiquitous in machine learning as it can effectively prevent overfitting and yield better generalization. The $\ell_1$-regularized training for Lasso estimators/sparse Gaussian graphical model (GMRF) estimation \cite{Tibshirani96,RWRY11} and $\ell_2$ weight decay \cite{Tikhonov43} on parameters are prototypical examples. In the context of deep learning, important instances include network pruning~\cite{wen2016learning,louizos2018learning}, which induces a sparse network structure, and network quantization \cite{yang2019quantization,courbariaux2015binaryconnect,bai2019proxquant}, which gives hard constraints so that parameters have only discrete values.


In many cases, the regularizer is non-smooth around some region (Consider $\ell_1$ norm at zero). Therefore, instead of using the gradient, one employs the subgradient of the objective function $F(\theta)$ in \eqref{eqn:problem}. Such a strategy, which is essentially adopted in modern machine learning libraries such as TensorFlow \cite{abadi2016tensorflow} and PyTorch \cite{paszke2019pytorch}, is problematic as it  may slow down convergence and result in oscillations. A simple idea to tackle this issue is to bypass the non-smoothness of a regularizer via its proximal operator. This idea is the basis of proximal gradient descent (PGD) methods, which first update the parameter using the gradient of the loss function $f(\theta)$ and then perform a proximal mapping of $\mathcal{R}(\theta)$. 
In the \emph{non-stochastic} case, the PGD with both convex and non-convex regularizers has been extensively studied in the literature \cite{reddi2016proximal,allen2017natasha,wang2018spiderboost,pham2019proxsarah,chen2020orthant}. 
In contrast, PGD in the \emph{stochastic} setting has been little explored. \cite{duchi11,ghadimi2016mini} consider PGD to solve the stochastic objectives with convex regularizers. Recently, \cite{xu2019stochastic} studies non-convex and non-smooth regularized problems for DC (difference of convex) functions and \cite{xu2019non} presents a non-asymptotic analysis for non-convex smooth loss and non-convex regularizers, which is the most general setting.


\begin{table}[t]
	\caption{Comparison among stochastic (or online) PGD for solving the problem \eqref{eqn:problem}.}
	\begin{center}
		\small{
			\begin{tabular}{l||l|l|l|l|l}
				\toprule
				Algorithm & \makecell{Non-convex~\\Loss} & \makecell{Non-convex~\\Regularizer} & Preconditioner & Momentum & \makecell{Convergence~\\Guarantee} \\
				\cmidrule{1-6}\morecmidrules\cmidrule{1-6}
				\textsc{AdaGrad} \cite{duchi11} & & & \textsc{AdaGrad} & & \cmark \\
				\cite{wang2018spiderboost,pham2019proxsarah,davis2019stochastic} & \cmark & & & \cmark & \cmark \\
				\cite{ghadimi2016mini} & \cmark & & \cmark & & \cmark \\
				\cite{xu2019stochastic} & \cmark & \cmark & \textsc{AdaGrad} &  & \cmark \\
				\cite{xu2019non} & \cmark & \cmark &  &  & \cmark \\
				Prox-SGD \cite{Yang2020ProxSGD} & \cmark &  & \cmark & \cmark &  \\
				\midrule
				\textsc{ProxGen} \textbf{(Ours)} & \cmark & \cmark & \cmark & \cmark & \cmark \\
				\bottomrule
			\end{tabular}
		}
	\end{center}
	\label{tab:proxopt_comparison}
\end{table}

All the aforementioned studies, however, focus either on limited settings (e.g. \cite{duchi11} only covers the update rule of \textsc{AdaGrad}) with convex regularizers only, or on pure vanilla gradient descent for non-convex regularizers. Hence, they cannot accommodate all advanced modern optimization algorithms with \emph{preconditioners}, such as adaptive gradient methods. The only exception is \textsc{Prox-SGD} \cite{Yang2020ProxSGD}, with the caveat that \textsc{Prox-SGD} update rule is \emph{not} exact proximal gradient descent. Moreover, the theoretical analysis in~\cite{Yang2020ProxSGD} only guarantees convergence, \emph{not how fast} Prox-SGD converges to stationary points, and furthermore this analysis is performed without considering the preconditioners. Table \ref{tab:proxopt_comparison} summarizes the previous studies and our work in terms of stochastic PGD.

In this paper, we propose an exact framework for stochastic proximal gradient methods with \emph{arbitrary} positive preconditioners and lower semi-continuous (possibly non-convex) regularizers. With our framework, our goal is to provide theoretical and empirical understanding of stochastic proximal gradient methods. Our main contributions are summarized as follows:
\begin{itemize}[leftmargin=4mm]
	\item We propose the first general family of stochastic proximal gradient methods, which we term \textsc{ProxGen}. With \textsc{ProxGen}, we present two novel update rules: (i) the closed-form proximal mappings of $\ell_q$ regularization $(0 \leq q \leq 1)$ for adaptive gradient methods such as \textsc{Adam}, and (ii) revised \textsc{ProxQuant} \cite{bai2019proxquant} which fixes a caveat of the original approach for quantization-specific regularizers.
	
	\item We analyze the convergence of the general \textsc{ProxGen} family and identify essential conditions for convergence guarantee. 
	We show that \textsc{ProxGen} enjoys the same convergence rate as vanilla \textsc{Sgd} under mild conditions. Moreover, our analysis encompasses several existing approaches as special cases.
	
	\item In terms of practice, we demonstrate the superiority of proximal methods over subgradient-based methods. Interestingly, our experiments show that proximal methods with non-convex regularizers are more effective than with convex regularizers for learning sparse deep models. 
	
\end{itemize}

\section{A Unified Framework for Stochastic Proximal Gradient Methods}\label{sec:main}

\begin{algorithm}[t]
	\caption{\textsc{ProxGen}: A \textbf{Gen}eral Stochastic \textbf{Prox}imal Gradient Method}
	\label{alg:proxgen_alg}{
		\begin{algorithmic}[1]
			\State {\bfseries Input:} Stepsize $\alpha_t$, $\{\rho_t\}_{t=1}^{t=T} \in [0, 1)$, regularization parameter $\lambda$, and small constant $0 < \delta <\!\!\!< 1$.
			\State {\bfseries Initialize:} $\theta_1 \in \mathbb{R}^{d}$, $m_0 = 0$, and $C_0 = 0$.
			\For{$t = 1, 2, \ldots, T$}
			\State{Draw a minibatch sample $\xi_t$ from $\mathbb{P}$}
			\Let{$g_t$}{$\nabla f(\theta_{t}; \xi_t)$} \Comment{Stochastic gradient at time $t$} 
			\Let{$m_t$}{$\rho_t m_{t-1} + (1 - \rho_t) g_t$} \Comment{First-order momentum estimate}\vspace{+0.1cm}
			\Let{$C_t$}{Preconditioner construction}
			\State{$\theta_{t+1} \in \argmin\limits_{\theta \in \Omega}\Big\{\langle m_t, \theta \rangle + \lambda \mathcal{R}(\theta) + \dfrac{1}{2\alpha_t}(\theta - \theta_{t})^\mathsf{T} \big(C_t + \delta I) (\theta - \theta_{t})\Big\}$}
			\EndFor
			\State {\bfseries Output:} $\theta_{T+1}$
	\end{algorithmic}}
\end{algorithm}

In this section, we present \textsc{ProxGen}, a general family of stochastic proximal gradient methods, and present both existing and novel instances as showcase examples in our family. Algorithm \ref{alg:proxgen_alg} describes the details of \textsc{ProxGen}. The update rule on line 8 of Algorithm \ref{alg:proxgen_alg} can be written more compactly:
\begin{align}
\theta_{t+1} & \in \argmin\limits_{\theta \in \Omega} \Big\{\langle m_t, \theta \rangle + \lambda \mathcal{R}(\theta) + \frac{1}{2\alpha_t} (\theta - \theta_t)^\mathsf{T} \Big(C_t + \delta I\Big) (\theta - \theta_t)\Big\} \nonumber\\
& = \mathrm{prox}_{\alpha_t \lambda \mathcal{R}(\cdot)}^{C_t + \delta I} \Big(\theta_t - \alpha_t (C_t + \delta I)^{-1} m_t\Big) \label{eqn:proxgen_update}
\end{align}

where the proximal operator in \eqref{eqn:proxgen_update} is defined as $\mathrm{prox}_h^{A}(z) = \argmin_x \{h(x) + \frac{1}{2} \|x - z\|_A^2\}$. In \textsc{ProxGen}, we allow both the loss and the regularizer to be non-convex. Based on our framework, we introduce possible examples according to the proper combinations of preconditioners $C_t$ and regularizers $\mathcal{R}(\cdot)$.

\paragraph{Existing Examples.} We first briefly recover some known examples in the \textsc{ProxGen} family.
\begin{itemize}[leftmargin=7mm]
	\item \textsc{AdaGrad} \cite{duchi11} is the first key instance of adaptive gradient methods where $C_t = \Big(\sum_{\tau=1}^{t} g_\tau g_\tau^\mathsf{T}\Big)^{1/2}$ and $\mathcal{R}(\theta) = \|\theta\|_1$. Any convex regularizer $\mathcal{R}(\cdot)$ is allowed. 
	\item The proximal Newton methods \cite{lee2012proximal} employ the exact Hessian preconditioner $C_t = \nabla^2 f(\theta_t)$ and $\mathcal{R}(\theta) = \|\theta\|_1$. In addition, we could replace the exact Hessian with suitable approximations, which yield proximal Newton-\emph{type} methods such as 
	quasi-Newton approximation \cite{becker2019quasi}, L-BFGS approximation \cite{liu1989limited}, and damping strategy adding a multiple of the identity to the Hessian.
\end{itemize}
Although the above examples enjoy good theoretical properties in convex settings, many of the modern practical optimization problems involve non-convex loss functions such as learning deep models. Moreover, it is known that non-convex regularizers yield better performance (also in terms of theory) than convex penalties in some applications (see \cite{fu1998penalized,park2011bridge,yang2017sparse+,yun2019trimming} and references therein). Considering this motivation and recent advanced optimizers, we arrive at the following new examples. 

\paragraph{Novel Examples.} Beyond the well-known methods above, \textsc{ProxGen} naturally introduces proximal versions of standard SGD techniques developed for solving unregularized problems for deep learning. The following examples are just a few instances that have not been explored so far, and \textsc{ProxGen} can cover a broader range of new examples depending on the combinations of preconditioners and regularizers.
\begin{itemize}[leftmargin=7mm]
	\item The \emph{proximal} version of \textsc{Adam} \cite{kingma15} with $\ell_q$ regularization is a possible example where $C_t = \sqrt{\beta C_{t-1} + (1 - \beta) g_t^2}$ with $\beta \in [0,1)$ and $\mathcal{R}(\theta) = \|\theta\|_q$ for $0 \leq q \leq 1$. We mainly validate the superiority of our novel \emph{proximal version} of \textsc{Adam} to the usual subgradient-based counterpart empirically in Section \ref{sec:exp}.
	\item We can also consider the \emph{proximal} version of \textsc{KFAC} \cite{martens2015optimizing} for neural networks. For an $L$-layer neural network, KFAC approximates the Fisher information matrix with layer-wise block diagonal structure where $l$-th diagonal block $C_{t,[l]}$ (for $1 \leq l \leq L$) corresponds to Kronecker-factored approximate Fisher matrix with respect to the parameters at $l$-th layer. The proximal version of K-FAC, which corresponds to $C_{t,[l]} = \mathbb{E}[\bm \delta_l \bm \delta_l^\mathsf{T}] \otimes \mathbb{E}[\bm a_{l-1} \bm a_{l-1}^\mathsf{T}]$ and $\mathcal{R}(\theta) = \|\theta\|_q$ where $\bm \delta_l$ is the gradient with respect to the output of $l$-th layer and $\bm a_{l-1}$ is the activation of $(l-1)$-th layer, could be another example. 
\end{itemize}

\paragraph{Relationship with \textsc{Prox-SGD} \cite{Yang2020ProxSGD}.} Proximal updates for stochastic preconditioned gradient methods have not been studied previously. \textsc{Prox-SGD} \cite{Yang2020ProxSGD} is the closest work. However, \textsc{Prox-SGD} is not an exact proximal approach and is significantly different from our \textsc{ProxGen} approach. Unlike \textsc{ProxGen} updating parameters by directly solving the quadratic subproblem \eqref{eqn:proxgen_update}, \textsc{Prox-SGD} updates the parameters in two stages: (i) solving the quadratic subproblem \emph{without} learning rate, 
then (ii) updating the parameters with the computed direction (i.e. $\widehat \theta_t - \theta_t$) by the learning rate $\alpha_t$ \eqref{eqn:proxsgd_update}.
\begin{align}
\widehat \theta_t & = \argmin\limits_{\theta \in \Omega} \Big\{\!\langle m_t, \theta \rangle + \lambda \mathcal{R}(\theta) + \frac{1}{2} (\theta - \theta_t)^\mathsf{T} \Big(C_t + \delta I\Big) (\theta - \theta_t)\!\Big\},~ 
\theta_{t+1} = \theta_t + \alpha_t (\widehat \theta_t - \theta_t) \label{eqn:proxsgd_update}
\end{align}
We also note that the two-stage update scheme of \textsc{Prox-SGD} might have some potential issues. For example, for $\ell_1$-regularized problems, the updated parameter $\theta_{t+1}$ in \eqref{eqn:proxsgd_update} \emph{might not achieve exact zero} (while $\widehat \theta_t$ can) whereas $\theta_{t+1}$ for \textsc{ProxGen} in \eqref{eqn:proxgen_update} can attain exact zero value according to the update rule \eqref{eqn:proxgenl1} in Section \ref{sec:examples}. Another potential caveat is that \textsc{Prox-SGD} might \emph{overestimate} the sparsity level. We provide details on this in Appendix, 
with an experiment comparing the support recovery of \textsc{Prox-SGD} and \textsc{ProxGen}.

\subsection{Examples of Proximal Mappings}\label{sec:examples}

In this section, we provide novel update rules for various regularizers $\mathcal{R}(\theta)$ in the \textsc{ProxGen} framework.

\paragraph{$\ell_q$ Regularization $(0 \leq q \leq 1)$ with Diagonal Preconditioners.} We consider the regularizer, $\mathcal{R}(\theta) = \lambda \sum_{j=1}^{p} |\theta_j|^q$ for $\theta \in \mathbb{R}^p$ with diagonal preconditioner matrix $C_t$. In the case of $C_t = I$ (i.e. vanilla gradient descent), it is known that there exists closed-form solutions for proximal mappings \cite{cao2013fast} for $q \in \{0, \frac{1}{2}, \frac{2}{3}, 1\}$, which is our basis for derivations. We denote the $i$-th coordinate of the vector $\theta_t$ as $\theta_{t,i}$ and the diagonal entry $[C_t]_{ii}$ as $C_{t,i}$ briefly.

\paragraph{$\bullet$~ $\ell_1$ regularization.} The proximal mappings of $\ell_1$ with preconditioner could be computed efficiently via soft-thresholding operators as
\begin{align}\label{eqn:proxgenl1}
\widehat \theta_{t,i} = \theta_{t,i} - \alpha_t \frac{m_{t,i}}{C_{t,i} + \delta}, \quad \theta_{t+1, i} = \mathrm{sign}\big(\widehat \theta_{t,i}\big) \Big(\big|\widehat \theta_{t,i}\big| - \frac{\alpha_t \lambda}{C_{t,i} + \delta}\Big)
\end{align}

\paragraph{$\bullet$~ $\ell_0$ regularization.} In case of $\ell_0$ regularization, we can compute the closed-form solutions with hard-thresholding operators as
\begin{align}\label{eqn:proxgenl0}
\widehat \theta_{t,i} = \theta_{t,i} - \alpha_t \frac{m_{t,i}}{C_{t,i} + \delta}, \quad \theta_{t+1, i} = 
\begin{cases}
\widehat \theta_{t,i}, & |\widehat \theta_{t,i}| > \sqrt{\frac{2\alpha_t \lambda}{C_{t,i} + \delta}}, \\
0, & |\widehat \theta_{t,i}| < \sqrt{\frac{2\alpha_t \lambda}{C_{t,i} + \delta}} \\
\{0, \widehat \theta_{t,i}\}, & |\widehat \theta_{t,i}| = \sqrt{\frac{2\alpha_t \lambda}{C_{t,i} + \delta}}
\end{cases}
\end{align}

We defer the closed-form proximal mappings for $\ell_{1/2}$ and $\ell_{2/3}$ regularization to Appendix. The important family of diagonal preconditioners is adaptive gradient methods such as \textsc{Adam}.

\textbf{Revising \textsc{ProxQuant} \cite{bai2019proxquant}.}~ 
\begin{table}[t]
	\caption{The comparison of \textsc{ProxQuant} \cite{bai2019proxquant} and \emph{revised} \textsc{ProxQuant} (ProxGen).}
	\label{tab:revised_prox}
	\begin{center}
		\begin{tabular}{c|c}
			\toprule
			\textsc{ProxQuant} & Revised \textsc{ProxQuant} (Ours) \\
			\cmidrule{1-2}\morecmidrules\cmidrule{1-2}
			$\mathrm{prox}_{\alpha_t \lambda \mathcal{R}(\cdot)}\Big(\theta_t - \alpha_t (C_t + \delta I)^{-1} m_t\Big)$ & $\mathrm{prox}_{\alpha_t \lambda \mathcal{R}(\cdot)}^{C_t + \delta I}\Big(\theta_t - \alpha_t (C_t + \delta I)^{-1} m_t\Big)$ \\
			\bottomrule
		\end{tabular}
	\end{center}
\end{table}
Recently, \textsc{ProxQuant} proposes novel regularizations for network quantization. Especially for binary quantization, the authors propose the W-shaped regularizer defined as $\mathcal{R}_{\mathrm{bin}}(\theta) = \|\theta - \mathrm{sign}(\theta)\|_1$ where $\mathrm{sign}(\theta)$ is applied on $\theta$ in an element-wise manner. With this regularizer, the main difference between \textsc{ProxQuant} and our \textsc{ProxGen} approach is shown in Table \ref{tab:revised_prox}. Note that \textsc{ProxQuant} (left in Table \ref{tab:revised_prox}) does not consider the effect of preconditioners when computing proximal mappings. Therefore, we revise the proximal update in \textsc{ProxQuant} by considering preconditioners in proximal mappings with \textsc{ProxGen} (right in Table \ref{tab:revised_prox}).

Moreover, we also propose generalized regularizers motivated by our $\ell_q$ regularization for $0 < q < 1$:
\begin{align}\label{eqn:quant_lq}
\mathcal{R}_{\mathrm{bin}}^{q}(\theta) = \|\theta - \mathrm{sign}(\theta)\|_q 
\end{align}
In terms of theory, \textsc{ProxQuant} \cite{bai2019proxquant} proves the convergence in \emph{deterministic} setting only when the regularizer is \emph{differentiable}, which is also guaranteed only for vanilla \textsc{Sgd}. Note that, in contrast, our \emph{revised} \textsc{ProxQuant} completely bridges this gap in theory by the theorem which we will show in next section in \emph{stochastic} optimization and provides the \emph{exact}  
update rule for solving problem \eqref{eqn:problem}. We also investigate the empirical differences of those two approaches in Section \ref{sec:exp}.

\section{Convergence Analysis}\label{sec:analysis}

In this section, we provide convergence guarantees for the \textsc{ProxGen} family. 
Our goal is to find an $\epsilon$-stationary point for the optimization problem \eqref{eqn:problem} where $\epsilon$ is the required precision.
For notational convenience, we assume that the regularization parameter $\lambda$ is incorporated into $\mathcal{R}(\theta)$ in \eqref{eqn:problem}. In order to guarantee the convergence under this setting, we should deal with the subdifferential defined as:
\begin{definition}[Fr\'echet Subdifferential]\label{def:subdiff}
	Let $\varphi$ be a real-valued function. The Fr\'echet subdifferential of $\varphi$ at $\bar{\theta}$ with $|\varphi(\bar{\theta})| < \infty$ is defined by
	\begin{align*}
	\widehat{\partial}\varphi(\bar{x}) \coloneqq \Big\{\theta^{*} \in \Omega ~\Big|~ \liminf\limits_{\theta \rightarrow \bar{\theta}} \frac{\varphi (\theta) - \varphi(\bar{\theta}) - \langle \theta^{*}, \theta - \bar{\theta} \rangle}{\|\theta - \bar{\theta}\|} \geq 0 \Big\}.
	\end{align*}
\end{definition}

To derive the convergence bound, we make the following mild assumptions:
\begin{enumerate}[itemsep=0em,leftmargin=0.5cm, itemindent=0.65cm,label=\textbf{(C-$\bf{\arabic*}$)}, ref=\textnormal{(C-$\arabic*$)},start=1]
	\item {($L$-smoothness)} The loss function $f$ is differentiable, $L$-smooth, and lower-bounded:\label{con:smooth}
	\begin{align*}
	\|\nabla f(x) - \nabla f(y)\| \leq L\|x - y\|
	\end{align*}
	\item {(Bounded variance)} The stochastic gradient $g_t = \nabla f(\theta_t; \xi)$ is unbiased and has the bounded variance: \label{con:var}
	\begin{align*}
	\mathbb{E}_{\xi} \big[\nabla f(\theta_t; \xi)\big] = \nabla f(\theta_t), \quad \mathbb{E}_{\xi}\big[\|g_t - \nabla f(\theta_t)\|^2\big] \leq \sigma^2. 
	\end{align*}
	\item (i) final step-vector is finite, (ii) the stochastic gradient is bounded, and (iii) the momentum parameter should be exponentially decaying: \label{con:mild}
	\begin{align*}
	\text{(i)}~~\|\theta_{t+1} - \theta_t\| \leq D, \qquad \text{(ii)}~~\|g_t\| \leq G, \qquad \text{(iii)}~~\rho_t = \rho_0 \mu^{t-1} 
	\end{align*}
	with $D, G > 0$ and $\rho_0, \mu \in [0, 1)$.
	\item {(\emph{Sufficiently positive-definite})} The minimum eigenvalue of effective spectrums should be uniformly lower bounded over all time $t$ by some strictly positive constant $\gamma$: \label{con:mineig}
	\begin{align*}
	\lambda_{\mathrm{min}}\big(\alpha_t(C_t + \delta I)^{-1}\big) \geq \gamma > 0
	\end{align*}
\end{enumerate}

Conditions \ref{con:smooth} and \ref{con:var} are standard in general non-convex optimization \cite{ghadimi2016mini,xu2019non,ghadimi2013stochastic,zaheer2018adaptive}. In addition, condition \ref{con:mild} is extensively studied in previous literature in the context of adaptive gradient methods \cite{kingma15,j.2018on,chen2018on}. Lastly, a similar condition to \ref{con:mineig} is also considered in \cite{Yang2020ProxSGD,lee2012proximal,chen2018on,yun2019stochastic}, and 
it can be easily satisfied in practice. More discussion on Condition \ref{con:mineig} is provided later. 

Since the loss function $f$ is assumed to be differentiable as in \ref{con:smooth}, we have, at stationary points, $\bm{0} \in \widehat{\partial}F(\theta) = \nabla f(\theta) + \widehat{\partial}\mathcal{R}(\theta)$, so the convergence criterion is slightly different from that of general non-convex optimization. 
Hence, we use the following convergence criterion $\mathbb{E}[\mathrm{dist}(\bm{0}, \widehat{\partial}F(\theta))] \leq \epsilon$ for an $\epsilon$-stationary point where $\mathrm{dist}(x, A)$ denotes the distance between a vector $x$ and a set $A$.
If no regularizer is considered ($\mathcal{R} = 0$), this criterion boils down to the one usually used in non-convex optimization, $\mathbb{E}[\|\nabla f(\theta)\|] \leq \epsilon$. We are now ready to state our main theorem for general convergence.

\begin{theorem}\label{thm:general_convergence}
	Let $\theta_a$ denote an iterate uniformly randomly chosen from $\{\theta_1, \cdots, \theta_T\}$. 
	Under the conditions \ref{con:smooth}, \ref{con:var}, \ref{con:mild}, \ref{con:mineig} with the initial stepsize $\alpha_0 \leq \frac{\delta}{3L}$ and non-increasing stepsize $\alpha_t$, \textsc{ProxGen}, Algorithm \ref{alg:proxgen_alg}, is guaranteed to yield
	\begin{align}
	\mathbb{E} \Big[\mathrm{dist}\big(\bm{0}, \widehat{\partial}F(\theta_a)\big)^2\Big] \leq \frac{Q_1\sigma^2}{T} \sum\limits_{t=0}^{T-1} \frac{1}{b_t} + \frac{Q_2 \Delta}{T} + \frac{Q_3}{T} 
	\end{align}
	where $\Delta = f(\theta) - f(\theta^{*})$ with optimal point $\theta^{*}$, and $b_t$ is the minibatch size at time $t$. The constants $\{Q_i\}_{i=1}^{3}$ on the right-hand side depend on the constants $\{\alpha_0, L, D, G, \rho_0, \mu, \gamma\}$, but not on $T$. 
\end{theorem}

From Theorem \ref{thm:general_convergence}, it can be seen that the appropriate minibatch size is important to ensure a good convergence rate. Various settings for the minibatch size could be employed for convergence guarantee (for example, dynamic minibatch size $b_t = t$), but in order to consider practical cases, we provide the following important corollary for \emph{constant minibatch size}.

\begin{corollary}[Constant Mini-batch]\label{cor:constant_minibatch}
	Under the same assumptions as in Theorem \ref{thm:general_convergence} with constant minibatch size $b_t = b = \Theta(T)$, we have $\mathbb{E} \big[\mathrm{dist}(\bm{0}, \widehat{\partial}F(\theta_a))^2\big] \leq \mathcal{O}\big(1/T\big)$
	and the total complexity is $\mathcal{O}(1/\epsilon^4)$ in order to have $\mathbb{E}\big[\mathrm{dist}\big(\bm 0, \widehat{\partial}F(\theta_a)\big)\big] \leq \epsilon$. 
\end{corollary}

\paragraph{Remarks.} Here we make several comments on our results and relationship with prior work.

\begin{itemize}[leftmargin=4mm]
	\item (On convergence of vanilla \textsc{Sgd})~ The very recent work \cite{xu2019non} analyzes the convergence of stochastic proximal methods for vanilla \textsc{Sgd} under similar assumptions to ours, which is a special case of our \textsc{ProxGen} framework, with $C_t = I$, $\rho_0 = 0$ and constant stepsize $\alpha_t = \alpha$. Our Corollary \ref{cor:constant_minibatch} enjoys the same convergence rate as Corollary 3 in \cite{xu2019non}. Note that our analysis also allows for \emph{non-increasing stepsizes}, which is much more practical in real problems. 
	
	\item (On convergence of EMA approaches)~ In terms of adaptive methods, \cite{zaheer2018adaptive} proves the convergence of \textsc{RMSprop} for general non-convex optimization, but this work considers unregularized training 
	(which corresponds to $\mathcal{R}(\theta) = 0$, constant stepsize $\alpha_t = \alpha$, and $\rho_t = 0$ in \textsc{ProxGen}). Our Corollary \ref{cor:constant_minibatch} achieves the exact same convergence rate for \textsc{RMSprop} as Corollary 3 in \cite{zaheer2018adaptive}. Also, our analysis can guarantee the convergence of \textsc{Adam}, which corresponds to non-zero $\rho_t$ with \emph{non-increasing stepsizes $\alpha_t$}.
	
	\item (On relationship with \textsc{Prox-SGD} \cite{Yang2020ProxSGD})~ \textsc{Prox-SGD}, as introduced in Section \ref{sec:main}, guarantees the convergence, but not how fast it converges. Moreover, this is proved without considering preconditioners. In contrast, \textsc{ProxGen} framework provides an \emph{exact} proximal update backed by detailed theoretical support. 
	
	\item (On examples satisfying \ref{con:mineig})~ Condition \ref{con:mineig} can be easily satisfied according to algorithmic details. 
	The popular optimization algorithm \textsc{Adam} \cite{kingma15} where $C_t = \sqrt{\beta C_{t-1} + (1 - \beta) g_t^2}$ for $\beta \in [0,1)$ satisfies this condition with $\gamma = \frac{\alpha}{G + \delta}$ under $\alpha_t = \alpha$. Detailed derivations and other examples are provided in the Appendix. 
	
	\item (On mini-batch condition in Corollary \ref{cor:constant_minibatch})~ In Corollary \ref{cor:constant_minibatch}, we believe that the condition on minibatch size $b = \Theta(T)$ is not stringent. As an example, consider a problem with sample size $n$ and minibatch size $b$ with maximum 200 epochs. 
	Then, the total iteration number $T$ should be $\Theta(\frac{200n}{b})$ resulting in $b = \Theta(\sqrt{n})$, which is practical in real cases.
	
	\item (On connections to second-order methods)~ Our analysis can provide guarantees for \emph{positive} second-order preconditioners as long as Condition \ref{con:mineig} is satisfied (The empirical Fisher information matrix \cite{martens2015optimizing} is one example). 
	Although second-order solvers generally enjoy very fast convergence under strongly convex loss \cite{lee2012proximal,zhang2019fast}, it could be understood that our analysis guarantees \emph{at least a sublinear rate for such second-order preconditioners} with less stringent conditions. 
\end{itemize}

\section{Experiments}\label{sec:exp}

We consider two important tasks for regularized training in deep learning communities: (i) training sparse neural networks and (ii) network quantization. Throughout our experiments, we consider \textsc{Adam} as a representative of \textsc{ProxGen} where $m_t = \rho_t m_{t-1} + (1 - \rho_t) g_t$ with constant decaying parameter $\rho_t = 0.9$ and $C_t = \sqrt{\beta C_{t-1} + (1 - \beta) g_t^2}$ with $\beta = 0.999$ in Algorithm \ref{alg:proxgen_alg}. The details on other hyperparameter settings for each experiment are provided in the Appendix.


\begin{figure}[t]
	\centering
	\subfigure[$\ell_1$ regularization]{\includegraphics[width=0.49\linewidth]{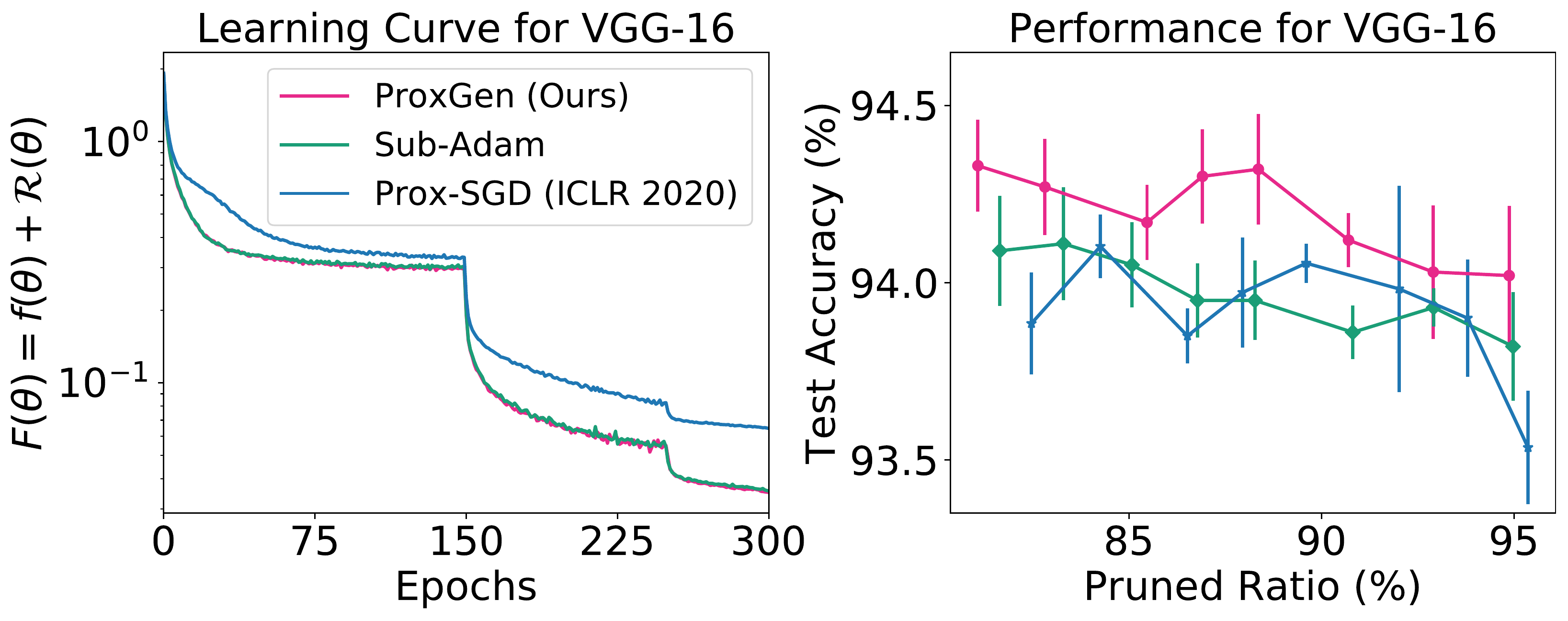}}
	\subfigure[$\ell_{2/3}$ regularization]{\includegraphics[width=0.49\linewidth]{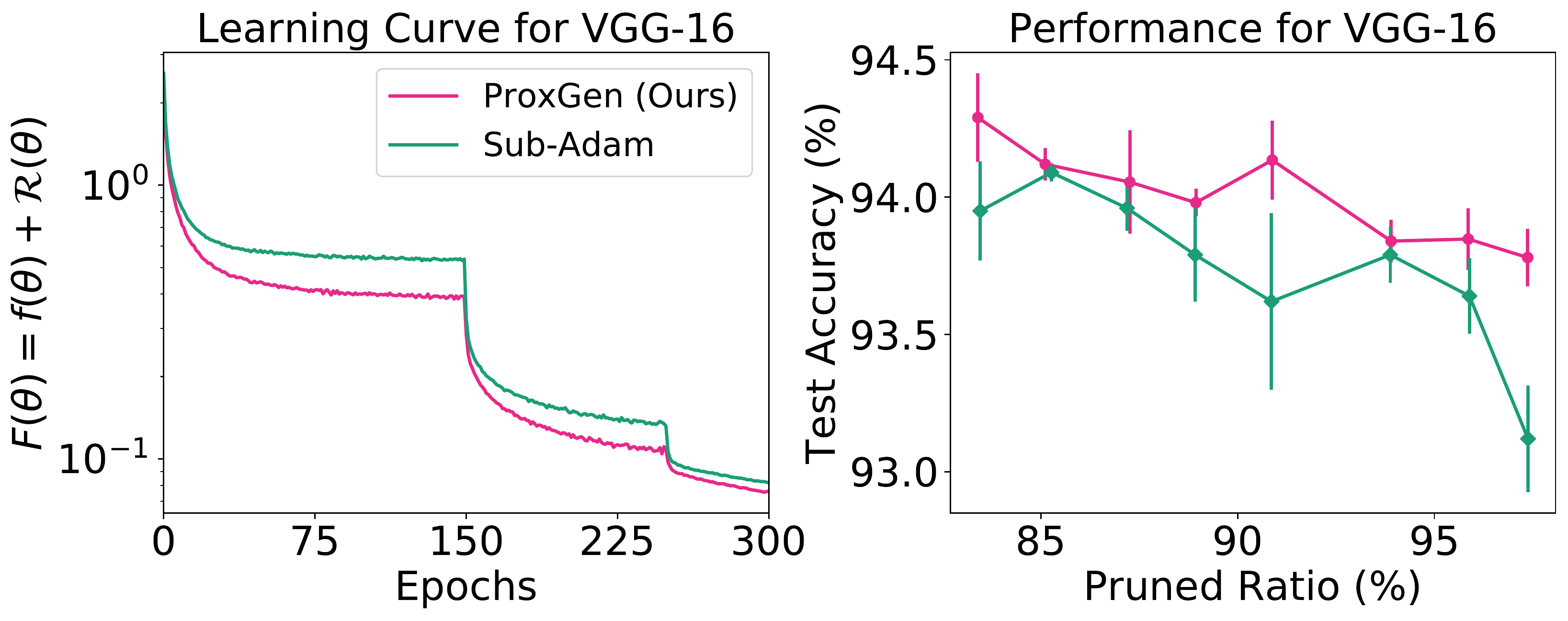}}
	\subfigure[$\ell_{1/2}$ regularization]{\includegraphics[width=0.49\linewidth]{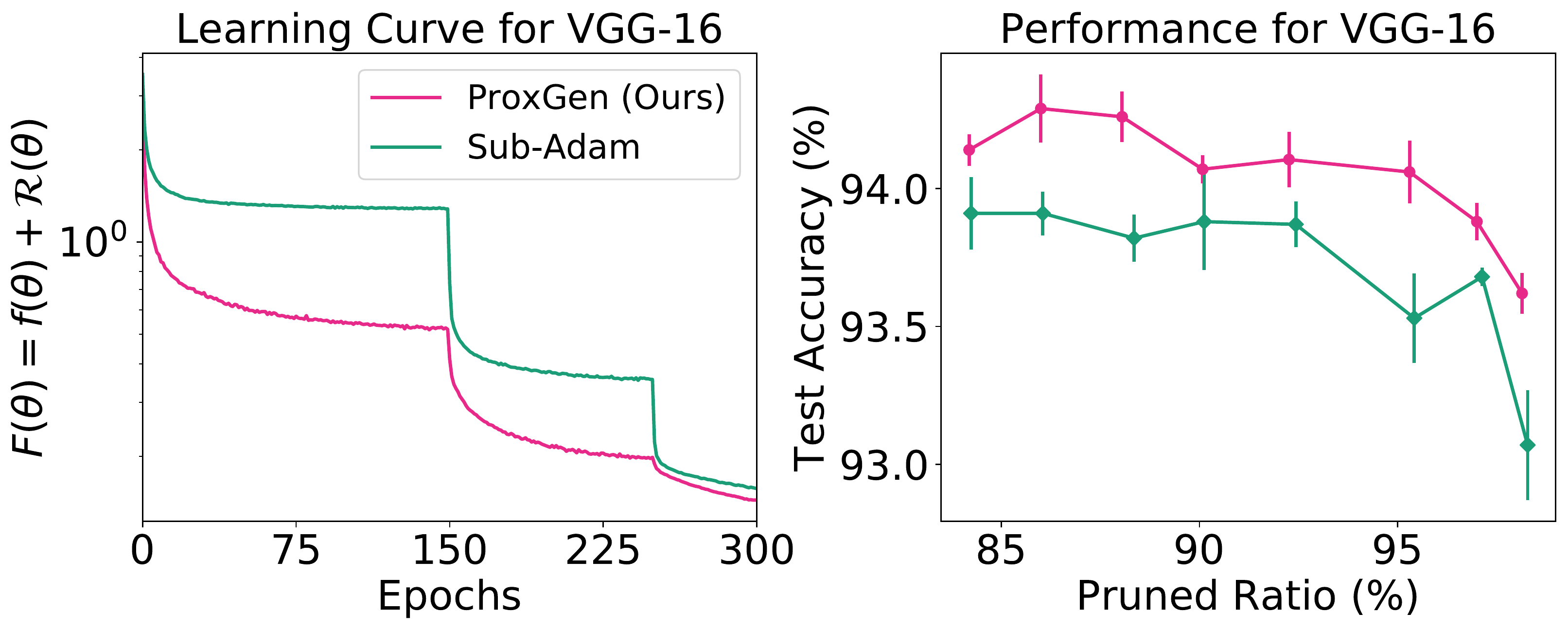}}
	\subfigure[$\ell_0$ regularization]{\includegraphics[width=0.24\linewidth]{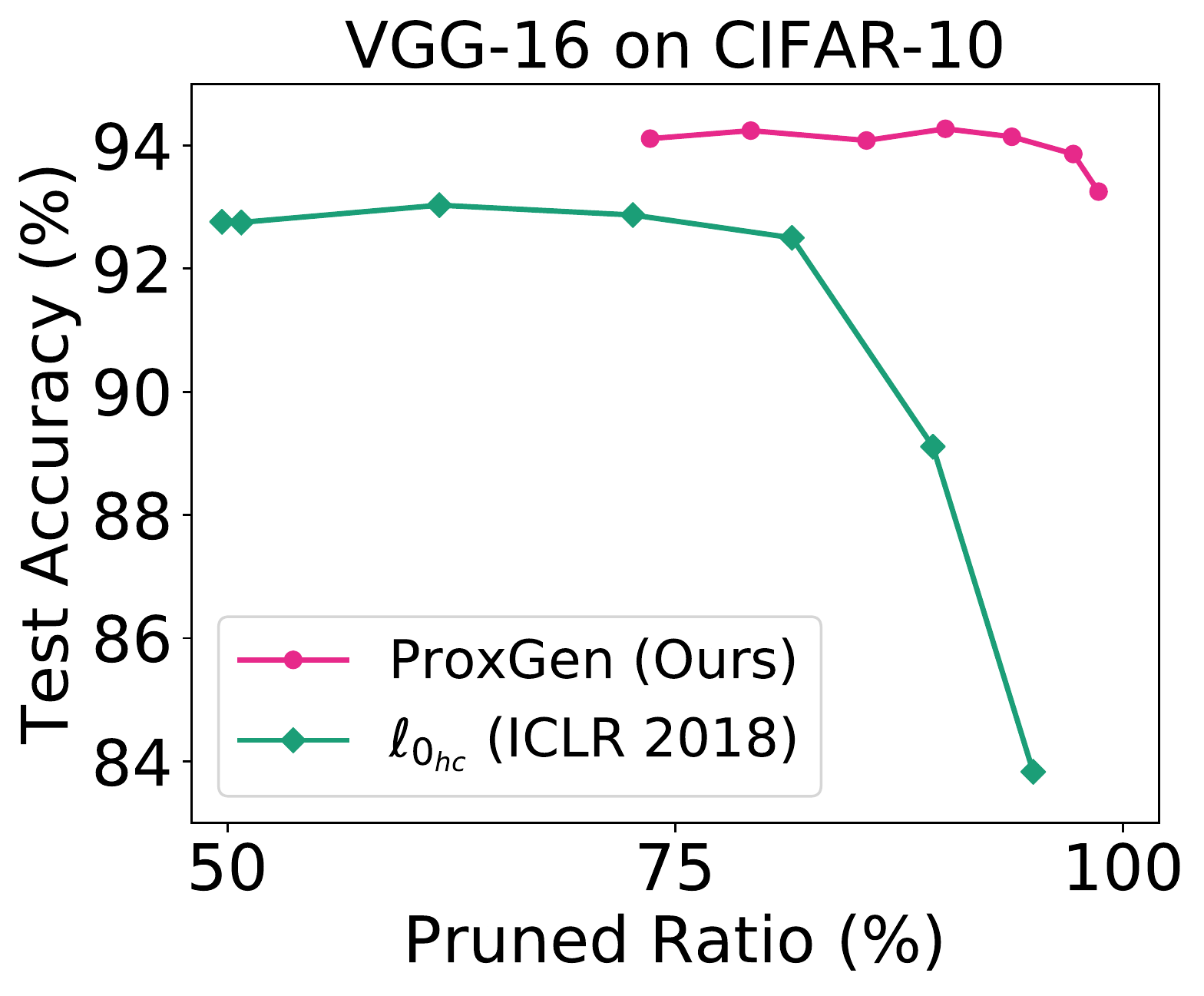}}
	\caption{Comparison for sparse VGG-16 on CIFAR-10 dataset.}
	\label{fig:lq_reg_vgg}
\end{figure}

\begin{figure}[t]
	\centering
	\subfigure[$\ell_1$ regularization]{\includegraphics[width=0.49\linewidth]{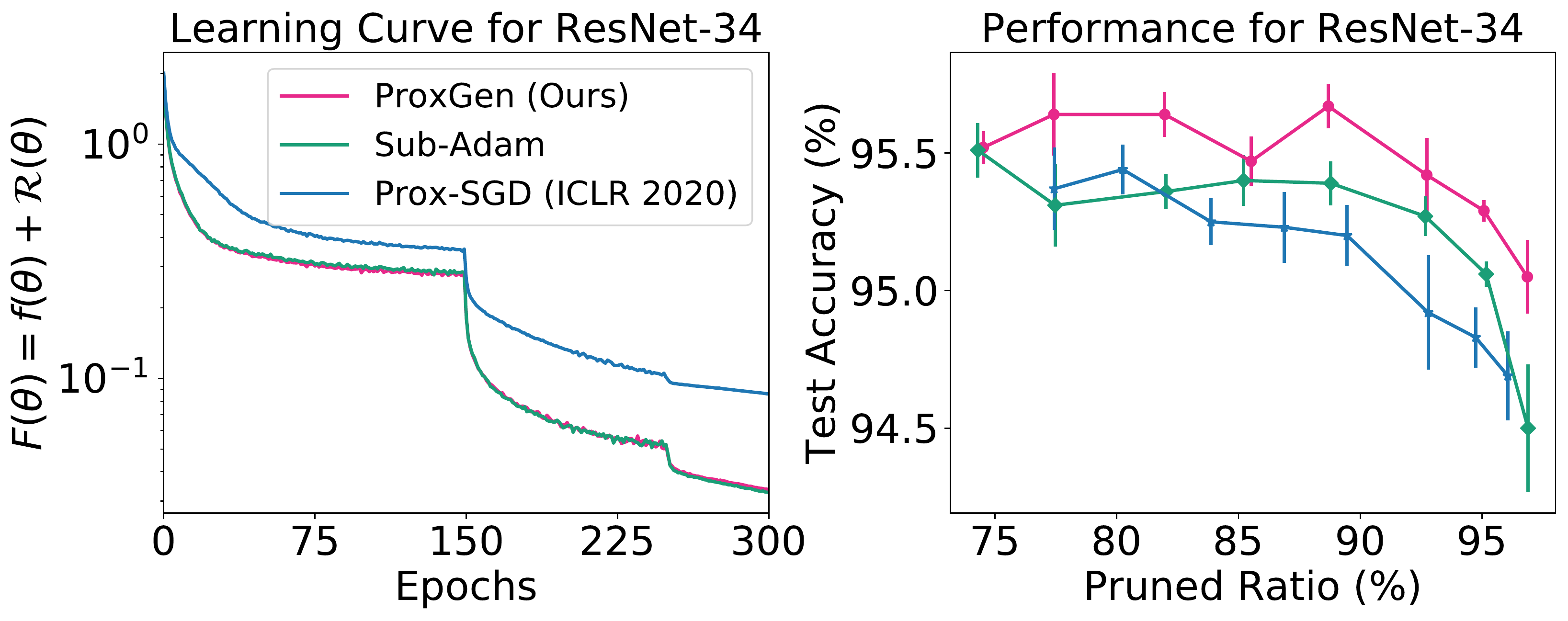}}
	\subfigure[$\ell_{2/3}$ regularization]{\includegraphics[width=0.49\linewidth]{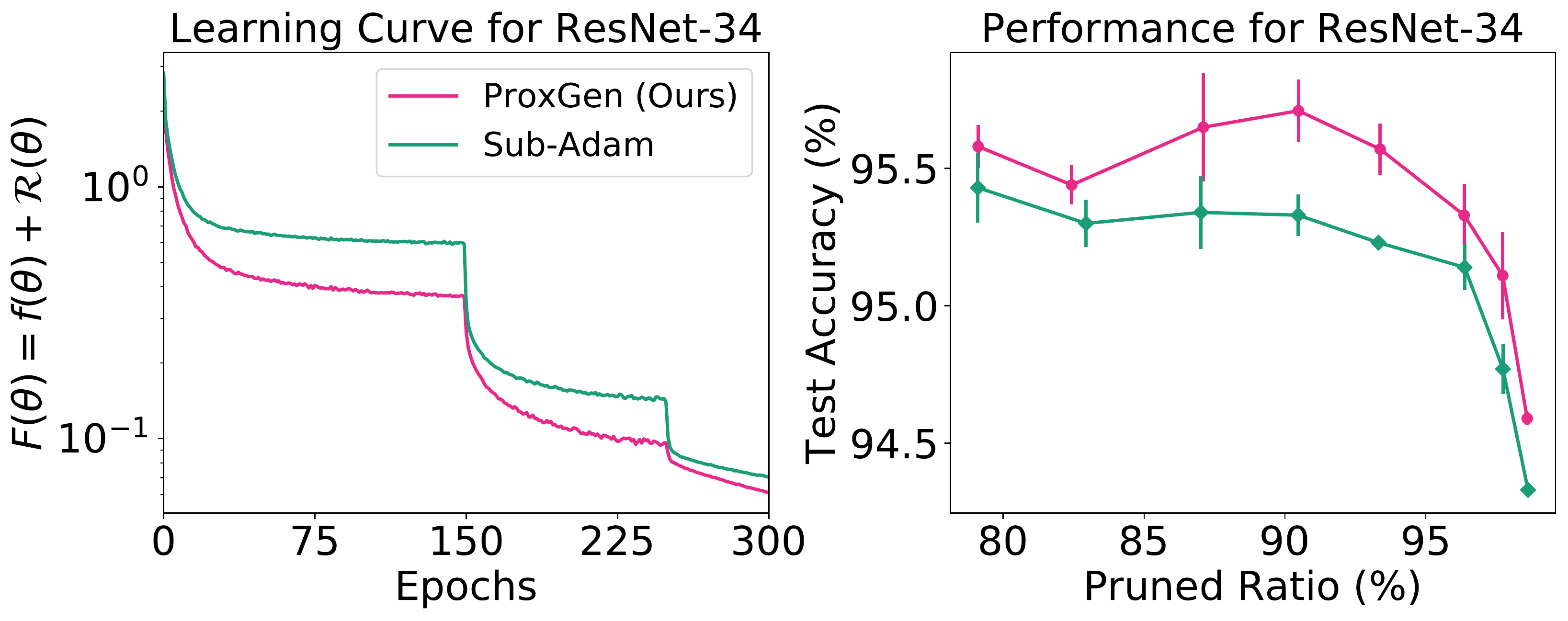}}
	\subfigure[$\ell_{1/2}$ regularization]{\includegraphics[width=0.49\linewidth]{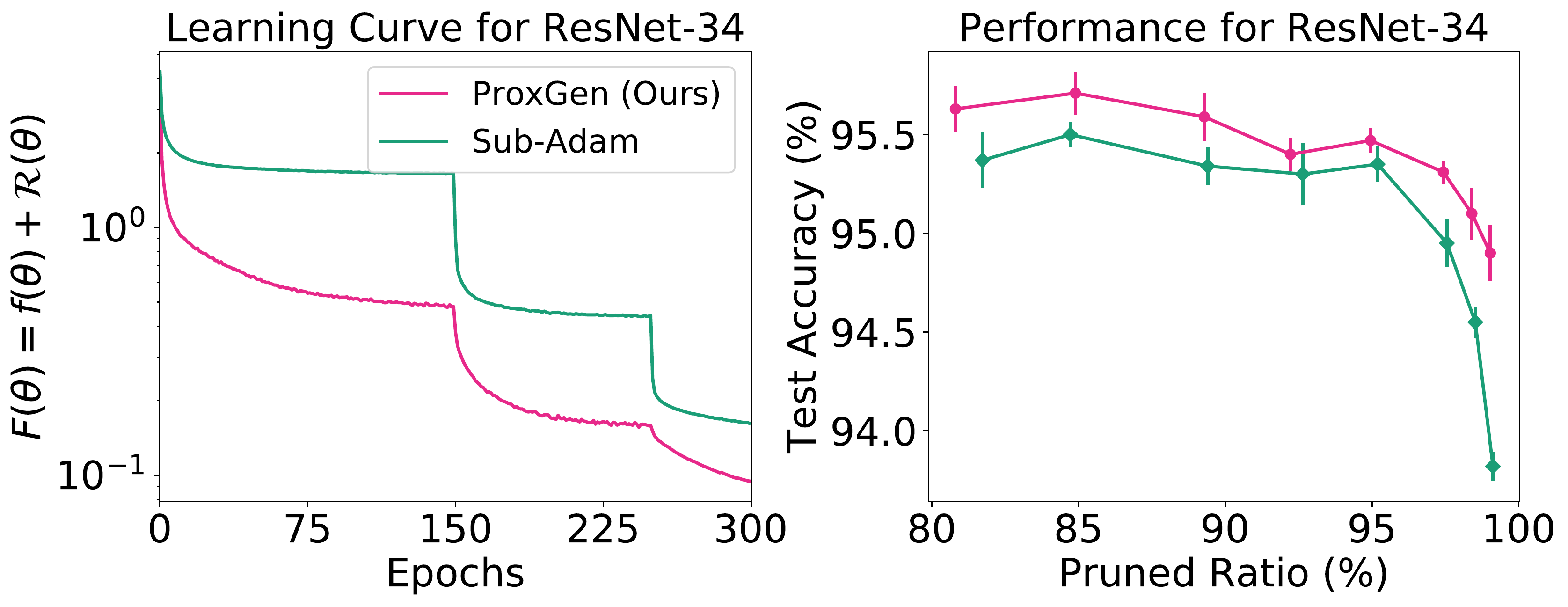}}
	\subfigure[$\ell_0$ regularization]{\includegraphics[width=0.225\linewidth]{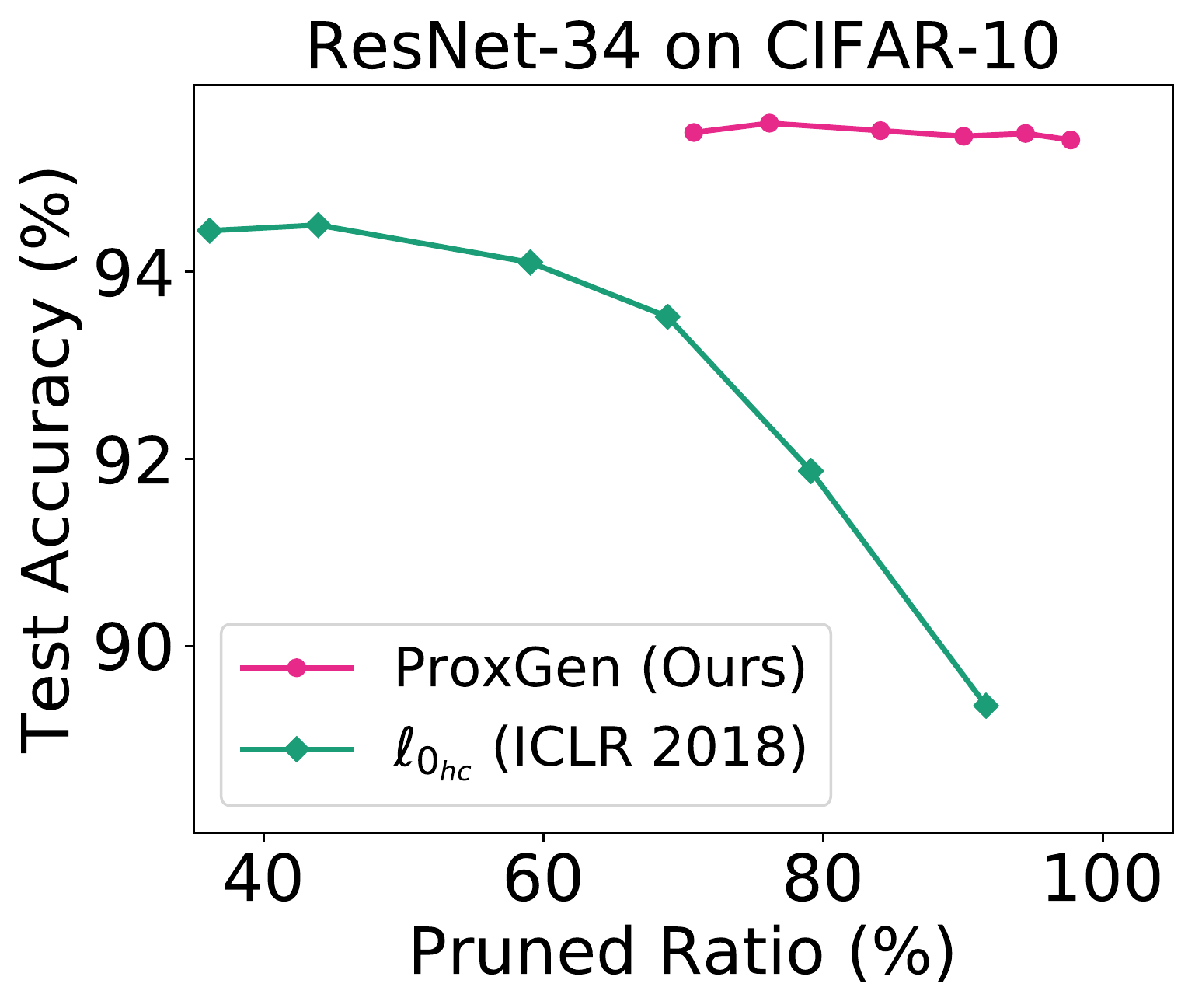}}
	\caption{Comparison for sparse ResNet-34 on CIFAR-10 dataset.}
	\label{fig:lq_reg_resnet}
\end{figure}

\paragraph{Training Sparse Neural Networks.} Motivated by the lottery ticket hypothesis \cite{frankle2018the}, we consider training VGG-16 \cite{simonyan2014very} and ResNet-34 \cite{he2016deep} on CIFAR-10 dataset using sparsity encouraging regularizers. Toward this, we consider the following objective function with $\ell_q$ regularization: $F(\theta) \coloneqq \mathbb{E}_{\xi \sim \mathbb{P}}[f(\theta; \xi)] + \lambda \sum_{j=1}^{p} |\theta_j|^q$
where $0 \leq q \leq 1$. We optimize the network parameters with the closed-form proximal mappings introduced in Section \ref{sec:examples}. In order to reflect the most practical training settings, we also consider the decoupled weight decay regularization \cite{loshchilov2018decoupled,zhang2018three}.

We compare \textsc{ProxGen} with subgradient methods and also include \textsc{Prox-SGD} \cite{Yang2020ProxSGD} as a baseline especially for $\ell_1$ regularization since \textsc{Prox-SGD} considers only convex regularizers. In \textsc{Prox-SGD}, the hand-crafted fine-tuned scheduling on $\alpha_t$ and $\rho_t$ is essential for fast convergence and good performance, but in our experiments we use standard settings $\rho_t = 0.9$ with step-decay learning rate scheduling for fair comparisons. For $\ell_0$ regularization, the problem \eqref{eqn:problem} cannot be optimized in a subgradient manner, so we compare \textsc{ProxGen} with another popular baseline, $\ell_{0_{hc}}$ \cite{louizos2018learning} which approximates the $\ell_0$-norm via hard-concrete distributions.

Figures \ref{fig:lq_reg_vgg} and \ref{fig:lq_reg_resnet} illustrate the results for VGG-16 and ResNet-34 respectively. In terms of convergence, \textsc{ProxGen} shows faster convergence than \textsc{Prox-SGD} \cite{Yang2020ProxSGD} for $\ell_1$ regularization, but there is no difference between \textsc{ProxGen} and subgradient methods. However, there are significant differences in convergence for non-convex regularizers $\ell_{1/2}$ and $\ell_{2/3}$, which get bigger as $q$ decreases. We believe this might be due to the fact that the $\ell_q$-norm derivative, $q/|\theta|^{1-q}$, is very large for tiny (but non-zero) $\theta$ for $0 < q < 1$. Meanwhile, the derivative of $|\theta|$ is nothing but the sign value regardless of size of $\theta$, hence the large gradient of $|\theta|^{q}$ may hinder convergence. 
The learning curves in Figure \ref{fig:lq_reg_vgg}-(b,c) and \ref{fig:lq_reg_resnet}-(b,c) empirically corroborate this phenomenon. 

In terms of performance, we can see that \textsc{ProxGen} consistently achieves better performance than baselines for both VGG-16 and ResNet-34 with similar or even better sparsity level. 
Importantly, \textsc{ProxGen} with $\ell_0$ outperforms $\ell_{0_{hc}}$ baseline by a great margin. This might be due to the design of $\ell_{0_{hc}}$, which approximates $\|\bm \theta\|_0 = \sum_{j=1}^{p} |\theta_j|_0$ with binary mask $z_j$ parameterized by learnable probability $\pi_j$ for each coordinate. Thus, the number of parameters to be optimized is doubled, which might make optimization harder. In constrast, \textsc{ProxGen} does not introduce additional parameters.


\begin{table}[t]
	\caption{Comparison for binary neural networks. The best performance in mean value is highlighted.}
	\centering
	{\footnotesize
		\begin{tabular}{c|c||c|c||c|c|c}
			\toprule
			\multicolumn{7}{c}{Test Error (\%)} \\
			\midrule
			\multicolumn{4}{c||}{Baselines} & \multicolumn{3}{c}{\textsc{ProxGen} (Ours)} \\
			\cmidrule{1-7}\morecmidrules\cmidrule{1-7}
			Model & \makecell{Full~\\Precision~\\(32-bit)} & \makecell{BinaryConnect~\\\cite{courbariaux2015binaryconnect}} & \makecell{\textsc{ProxQuant}~\\\cite{bai2019proxquant}} & \makecell{Revised~\\ProxQuant~\\$\ell_1$} & \makecell{Revised~\\ProxQuant~\\$\ell_{2/3}$} & \makecell{Revised~\\ProxQuant~\\$\ell_{1/2}$} \\
			\midrule
			ResNet-20 & 8.06 & 9.54 $\pm$ 0.03 & \textbf{9.35} $\pm$ 0.13 & 9.50 $\pm$ 0.12 & 9.72 $\pm$ 0.06 & 9.78 $\pm$ 0.18 \\
			ResNet-32 & 7.25 & 8.61 $\pm$ 0.27 & 8.53 $\pm$ 0.15 & 8.29 $\pm$ 0.07 & \textbf{8.22} $\pm$ 0.05 & 8.43 $\pm$ 0.15 \\
			ResNet-44 & 6.96 & 8.23 $\pm$ 0.23 & 7.95 $\pm$ 0.05 & \textbf{7.68} $\pm$ 0.07 & 7.91 $\pm$ 0.08 & 7.90 $\pm$ 0.13 \\
			ResNet-56 & 6.54 & 7.97 $\pm$ 0.22 & 7.70 $\pm$ 0.06 & \textbf{7.52} $\pm$ 0.18 & 7.60 $\pm$ 0.09 & 7.61 $\pm$ 0.12 \\
			\bottomrule
		\end{tabular}
	}
	\label{tab:quant_results}
\end{table}

\paragraph{Training Binary Neural Networks.}~ In the second set of experiments, we consider the network quantization constraining the parameters to some set of discrete values which is a key approach for model compression. We evaluate our revised \textsc{ProxQuant} in Table \ref{tab:revised_prox} with extended regularization \eqref{eqn:quant_lq} in Section \ref{sec:examples}. We consider the following objective function with quantization-specific regularizers: $F(\theta) \coloneqq \mathbb{E}_{\xi \sim \mathbb{P}}[f(\theta; \xi)] + \lambda \sum_{j=1}^{p} |\theta_j - \sign(\theta_j)|^q$ where $0 \leq q \leq 1$. For comparisons, we quantize ResNet \cite{he2016deep} on CIFAR-10 dataset and follow the same experiment settings as in \textsc{ProxQuant} \cite{bai2019proxquant}.

Table \ref{tab:quant_results} presents the results. For all $q$ values, revised \textsc{ProxQuant} consistently outperforms the baselines except for ResNet-20, which implies \textsc{ProxGen} may work better for larger networks. As such, our generalized regularizers \eqref{eqn:quant_lq} contribute to one of the state-of-the-art optimization-based methods in network quantization. 
Notably, revised \textsc{ProxQuant} $\ell_1$ greatly outperforms \textsc{ProxQuant} baseline while these two approaches differ only in update rules (see Table \ref{tab:revised_prox}). Hence, we can conclude that revised \textsc{ProxQuant} based on our \textsc{ProxGen} provides an \emph{exact} proximal update and furthermore yields more generalizable solutions. In our experience, revised \textsc{ProxQuant} $\ell_0$ shows little degradation in performance, so we do not include this result. However, revised \textsc{ProxQuant} $\ell_0$ shows superiority to baselines for language modeling, whose preliminary results are deferred to the Appendix.

\section{Conclusion}

In this work, we proposed \textsc{ProxGen}, the first general family of stochastic proximal gradient methods. Within our framework, we presented novel examples of proximal versions of standard SGD approaches, including a proximal version of \textsc{Adam}. We analyzed the convergence of the whole \textsc{ProxGen} family and showed that \textsc{ProxGen} can encompass the results of several previous studies. We also demonstrated that \textsc{ProxGen} empirically outperforms subgradient-based methods for popular deep learning problems. As future work, we plan to study efficient approximations of proximal mappings for structured regularizers such as $\ell_1/\ell_q$ norms with preconditioners.

\section*{Broader Impact}
Our work proposes a general framework for stochastic proximal gradient descent for deep learning. Our framework \textsc{ProxGen} would benefit both researchers and practitioners in machine learning. From a theoretical perspective, \textsc{ProxGen} provides the first exact proximal gradient descent updates for a wide class of regularized optimization 
problems, and opens up an avenue of research in studying various combinations of regularizers and preconditioners, as well as devising efficient 
computations for proximal mappings. From a practical standpoint, \textsc{ProxGen} enables machine learning practitioners to solve a wide class of regularized deep learning problems while enjoying faster convergence and better generalization. 
We do not believe that our research puts anyone at disadvantage. Our framework is backed by theoretical guarantees, provided that some mild conditions are satisfied. If these conditions were violated, 
the algorithm might not converge properly, but our experimental results show that \textsc{ProxGen} converges well even for ReLU networks in practice.
Our framework is a general purpose optimization approach to solve deep learning problems. As such, it does not target the identification nor uses bias in datasets.

\bibliographystyle{unsrt}
\bibliography{proxgen}

\clearpage
\small
\appendix

\section*{Supplementary Materials}

\section{Comparison for Support Recovery}
For support recovery to compare \textsc{ProxGen} and \textsc{Prox-SGD}, we generate simple Lasso simulations with problem dimension $p = 500$ and $n = 100$ data samples. The number of non-zero entries in true parameter vector $\theta^{*} \in \mathbb{R}^p$ is set to $10$. The design matrix $X \in \mathbb{R}^{n \times p}$ is generated from standard Gaussian distribution $\mathcal{N}(0, 1)$ and we randomly assign $+1$ or $-1$ for the non-zero value in true parameter at random 10 coordinates. The response variable $y \in \mathbb{R}^n$ is generated with small noise by $y = X\theta^{*} + \epsilon$ where $\epsilon \sim \mathcal{N}(0, 0.05^2)$. For both \textsc{ProxGen} and \textsc{Prox-SGD}, we employ \textsc{Adam} for preconditioner matrix $C_t$.

Under this setting, we simulate the support recovery using \textsc{ProxGen} and \textsc{Prox-SGD} with different two initialization methods: (i) random initialization and (ii) zero initialization. In Section \ref{sec:main}, we note that Prox-SGD have two potential caveats: (i) Prox-SGD might not achieve the exact zero signals and (ii) Prox-SGD might overestimate the sparsity level. In random initialization, it can be seen in Figure \ref{fig:lasso_random} that \textsc{Prox-SGD} could not achieve the exact zero value, which corroborates our first observation. To address this issue of \textsc{Prox-SGD}, we also conduct this simulation with \emph{zero} initialization. Interestingly, in this case, we can see in Figure \ref{fig:lasso_zero} that \textsc{Prox-SGD} shows zero signals for all coordinates, which is our second observation. This might be due to the fact that $\widehat \theta_t$ in \eqref{eqn:proxsgd_update} is always zero since the subproblem in \eqref{eqn:proxsgd_update} do not consider the learning rate, which might overestimate the sparsity level. Hence, the subsequent iterate $\theta_{t+1}$ would be always zero since we initialize the parameters with zero values. On the other hand, our \textsc{ProxGen} correctly recover the support in both cases.

\section{Details on Experimental Settings}

\paragraph{Sparse Neural Networks.} To reflect the most practical training settings, we first tune the weight-decay parameter $\zeta$ without $\ell_q$ regularizers. For weight-decay coefficients, we consider the candidates $\zeta \in \{0.001, 0.002, 0.005, 0.01, 0.02, 0.05, 0.1, 0.2, 0.5\}$ for $\zeta$ and the best $\zeta$ value is $0.2$ for both networks VGG-16 and ResNet-34 in our experience. After tuning weight-decay coefficient $\zeta$, we consider both decoupled weight decay \cite{loshchilov2018decoupled} and $\ell_q$ regularization whose detail update rule is described in Algorithm \ref{alg:proxgenw_alg}. For all comparison methods except $\ell_{0_{hc}}$, the recommended stepsize $\alpha_t = 0.001$ is employed, but we tune this stepsize for $\ell_{0_{hc}}$ baseline. We consider a broad range of regularization parameters for all methods: $\lambda \in \{0.001, 0.002, 0.005, 0.01, 0.02, \cdots, 1.0, 2.0, 5.0\}$. With these hyperparameter settings, we consider the total $300$ epochs and divide the learning rate at $150$-th and $250$-th epoch by $10$.

\paragraph{Binary Neural Networks.} In this experiment, we follow the same experimental settings in baseline \textsc{ProxQuant} \cite{bai2019proxquant}. We first pre-train ResNet-$\{$20, 32, 44, 56$\}$ with full-precision and initialize the network parameters with these pre-trained weights. Then, we consider the total $300$ epochs and hard-quantize the networks at $200$-th epoch (i.e. quantizing the weight parameters to $+1$ or $-1$). We employ the homotopy method introduced in \cite{bai2019proxquant}: annealing the regularization paramter $\lambda$ as $\lambda_{\text{epoch}} = \lambda \times \text{epoch}$. For initial value of $\lambda$, we use $\lambda = 10^{-8}$ or $\lambda = 5\cdot 10^{-8}$ for all ResNet architecture. We use the constant stepsize $\alpha_t = 0.01$ as recommended in \cite{bai2019proxquant}.

Here, we introduce preliminary results of revised \textsc{ProxQuant} $\ell_0$ on language modeling. For this experiment, we train one hidden layer LSTM with embedding dimension 300 and 300 hidden units according to \cite{bai2019proxquant}. First, we pre-train the full-precision LSTM and initialize the network with pre-trained weights. We consider the total 80 epochs and divide the learning rate by $1.2$ if the validation loss does not decrease. Table \ref{tab:lstm_quant_results} shows the preliminary results and revised \textsc{ProxQuant} $\ell_0$ is superior to the \textsc{ProxQuant} baseline in this task.

\begin{table}[b]
	\caption{Preliminary results on revised \textsc{ProxQuant} $\ell_0$ for LSTM models.}
	\centering
	\begin{tabular}{c|c}
		\toprule
		Algorithm & Test Perplexity \\
		\midrule
		Full-precision (32-bit) & 88.5 \\
		\midrule 
		BinaryConnect \cite{courbariaux2015binaryconnect} & 372.2 \\
		\textsc{ProxQuant} \cite{bai2019proxquant} & 288.5 \\
		\cmidrule{1-2}\morecmidrules\cmidrule{1-2}
		revised \textsc{ProxQuant} $\ell_0$ (Ours) & \textbf{223.4} \\
		\bottomrule
	\end{tabular}
	\label{tab:lstm_quant_results}
\end{table}

\section{Derivations for Proximal Mappings}
Here, we derive the concrete update rule for $\ell_q$ regularization with \emph{diagonal} preconditioners as introduced in Section \ref{sec:examples}. 

\paragraph{$\ell_{1/2}$ regularization.} First, we review the closed-form proximal mappings for $\ell_{1/2}$ regularization of vanilla \textsc{Sgd}. First, we consider the following one-dimensional program:
\begin{align}\label{eqn:vanilla_l12}
\widehat x = \argmin_x \{(x-z)^2 + \lambda |x|^{1/2}\}
\end{align}
For the program \eqref{eqn:vanilla_l12}, it is known that the closed-form solution exists \cite{cao2013fast} as
\begin{align}
\widehat x =
\begin{cases}
\frac{2}{3}|z| \Big(1 + \cos\big(\frac{2}{3}\pi - \frac{2}{3} \varphi_\lambda(z)\big)\Big) & \text{ if } z > p(\lambda) \\
0 & \text{ if } |z| \leq p(\lambda) \\
-\frac{2}{3}|z| \Big(1 + \cos\big(\frac{2}{3}\pi - \frac{2}{3} \varphi_\lambda(z)\big)\Big) & \text{ if } z < -p(\lambda)
\end{cases}
\end{align}
where $\varphi_\lambda(z) = \arccos\Big(\frac{\lambda}{8}\big(\frac{|z|}{3}\big)^{-3/2}\Big)$ and $p(\lambda) = \frac{\sqrt[3]{54}}{4}(\lambda)^{2/3}$. Based on this closed-form solution, we derive \textsc{ProxGen} for $\ell_{1/2}$ regularization with diagonal preconditioners. By \eqref{eqn:proxgen_update}, we have
\begin{align}
\widehat \theta_t & = \theta_t - \alpha_t (C_t + \delta I)^{-1} m_t \\
\theta_{t+1} & \in \mathrm{prox}_{\alpha_t \lambda \mathcal{R}(\cdot)}^{C_t + \delta I}(\widehat \theta_t) \\
& = \argmin\limits_{\theta} \Big\{\frac{1}{2}\|\theta - \widehat \theta_t\|_{C_t + \delta I}^2 + \lambda \sum\limits_{j=1}^{p} |\theta_j|^{1/2}\Big\} \label{eqn:l12_program}
\end{align}
Since the program \eqref{eqn:l12_program} is coordinate-wise decomposable (since the preconditioner matrix $C_t$ is diagonal), we can split \eqref{eqn:l12_program} into
\begin{align*}
\theta_{t+1, i} & = \argmin\limits_{\theta_i} \Big\{\frac{1}{2} (C_{t,i} + \delta) (\theta_i - \widehat \theta_{t,i})^2 + \alpha_t\lambda |\theta_i|^{1/2}\Big\} \\
& = \argmin\limits_{\theta_i} \Big\{(\theta_i - \widehat \theta_{t,i})^2 + \frac{2\alpha_t \lambda}{C_{t,i} + \delta} |\theta_i|^{1/2}\Big\}
\end{align*}
for the $i$-th coordinate. From \eqref{eqn:vanilla_l12}, we can derive
\begin{align*}
\theta_{t+1,i} = 
\begin{cases}
\frac{2}{3}|\widehat \theta_{t,i}| \Big(1 + \cos\big(\frac{2}{3}\pi - \frac{2}{3} \varphi_\lambda(\widehat \theta_{t,i})\big)\Big) & \text{ if } \widehat \theta_{t,i} > p(\lambda) \\
0 & \text{ if } |\widehat \theta_{t,i}| \leq p(\lambda) \\
-\frac{2}{3}|\widehat \theta_{t,i}| \Big(1 + \cos\big(\frac{2}{3}\pi - \frac{2}{3} \varphi_\lambda(\widehat \theta_{t,i})\big)\Big) & \text{ if } \widehat \theta_{t,i} < -p(\lambda)
\end{cases}
\end{align*}
where 
\begin{align*}
\varphi_\lambda(\widehat \theta_{t,i}) = \arccos\Big(\frac{\alpha_t\lambda}{4 (C_{t,i} + \delta)}\big(\frac{|\widehat \theta_{t,i}|}{3}\big)^{-3/2}\Big), \quad p(\lambda) = \frac{\sqrt[3]{54}}{4} \Big(\frac{2\alpha_t \lambda}{C_{t,i} + \delta}\Big)^{2/3}.
\end{align*}

\paragraph{$\ell_{2/3}$ regularization.} Now, we provide the closed-form solutions for proximal $\ell_{2/3}$ mappings with diagonal preconditioners. Similar to $\ell_{1/2}$ regularization, we start from the closed-form solutions of the following program:
\begin{align}\label{eqn:vanilla_l23}
\widehat x = \argmin_x \{(x-z)^2 + \lambda|x|^{2/3}\}
\end{align}
The closed-form solution for the program \eqref{eqn:vanilla_l23} is known to be
\begin{align}
\widehat x = 
\begin{cases}
\Bigg(\frac{|A| + \sqrt{\frac{2|z|}{|A|} - |A|^2}}{2}\Bigg)^3 & \text{ if } z > \frac{2}{3} \sqrt[4]{3\lambda^3} \\
0 & \text{ if } |z| \leq \frac{2}{3}\sqrt[4]{3\lambda^3} \\
-\Bigg(\frac{|A| + \sqrt{\frac{2|z|}{|A|} - |A|^2}}{2}\Bigg)^3 & \text{ if } z < -\frac{2}{3} \sqrt[4]{3\lambda^3}
\end{cases}
\end{align}
where
\begin{align}
|A| = \frac{2}{\sqrt{3}}\lambda^{1/4} \Big(\mathrm{cosh}\big(\frac{\phi}{3}\big)\Big)^{1/2}, \quad \phi = \mathrm{arccosh}\Big(\frac{27z^2}{16}\lambda^{-3/2}\Big)
\end{align}
Based on this formulation, we derive the closed-form proximal mappings with diagonal preconditioner $C_t$. By \eqref{eqn:proxgen_update}, we have
\begin{align}
\widehat \theta_t & = \theta_t - \alpha_t (C_t + \delta I)^{-1} m_t \\
\theta_{t+1} & \in \mathrm{prox}_{\alpha_t \lambda \mathcal{R}(\cdot)}^{C_t + \delta I}(\widehat \theta_t) \\
& = \argmin\limits_{\theta} \Big\{\frac{1}{2}\|\theta - \widehat \theta_t\|_{C_t + \delta I}^2 + \lambda \sum\limits_{j=1}^{p} |\theta_j|^{2/3}\Big\} \label{eqn:l23_program}
\end{align}
As in $\ell_{1/2}$ case, the program \eqref{eqn:l23_program} is coordinate-wise separable, so it suffices to solve the sub-problems for each coordinate as
\begin{align*}
\theta_{t+1, i} & = \argmin\limits_{\theta_i} \Big\{\frac{1}{2}(C_{t,i} + \delta) (\theta_i - \widehat \theta_i)^2 + \alpha_t \lambda |\theta_i|^{2/3}\Big\} \\
& = \argmin\limits_{\theta_i} \Big\{(\theta_i - \widehat \theta_{t,i})^2 + \frac{2\alpha_t \lambda}{C_{t,i} + \delta}|\theta_i|^{2/3}\Big\}
\end{align*}
From \eqref{eqn:vanilla_l23}, we can derive
\begin{align*}
\theta_{t+1, i} = 
\begin{cases}
\Bigg(\frac{|A| + \sqrt{\frac{2|\widehat \theta_{t,i}|}{|A|} - |A|^2}}{2}\Bigg)^3 & \text{ if } \widehat \theta_{t,i} > \frac{2}{3} \sqrt[4]{3\lambda^3} \\
0 & \text{ if } |\widehat \theta_{t,i}| \leq \frac{2}{3}\sqrt[4]{3\lambda^3} \\
-\Bigg(\frac{|A| + \sqrt{\frac{2|\widehat \theta_{t,i}|}{|A|} - |A|^2}}{2}\Bigg)^3 & \text{ if } \widehat \theta_{t,i} < -\frac{2}{3} \sqrt[4]{3\lambda^3}
\end{cases}
\end{align*}
where
\begin{align*}
|A| = \frac{2}{\sqrt{3}}\Big(\frac{2\alpha_t \lambda}{C_{t,i} + \delta}\Big)^{1/4} \Big(\mathrm{cosh}\big(\frac{\phi}{3}\big)\Big)^{1/2}, \quad \phi = \mathrm{arccosh}\Big(\frac{27\widehat \theta_{t,i}^2}{16}\Big(\frac{2\alpha_t \lambda}{C_{t,i} + \delta}\Big)^{-3/2}\Big)
\end{align*}

Although the derivations look little complicated for both cases, we emphasize that both two closed-form solutions can be efficiently implemented in a GPU-friendly manner.

\begin{figure}[t]
	\centering
	\subfigure[True parameter]{\includegraphics[width=0.32\linewidth]{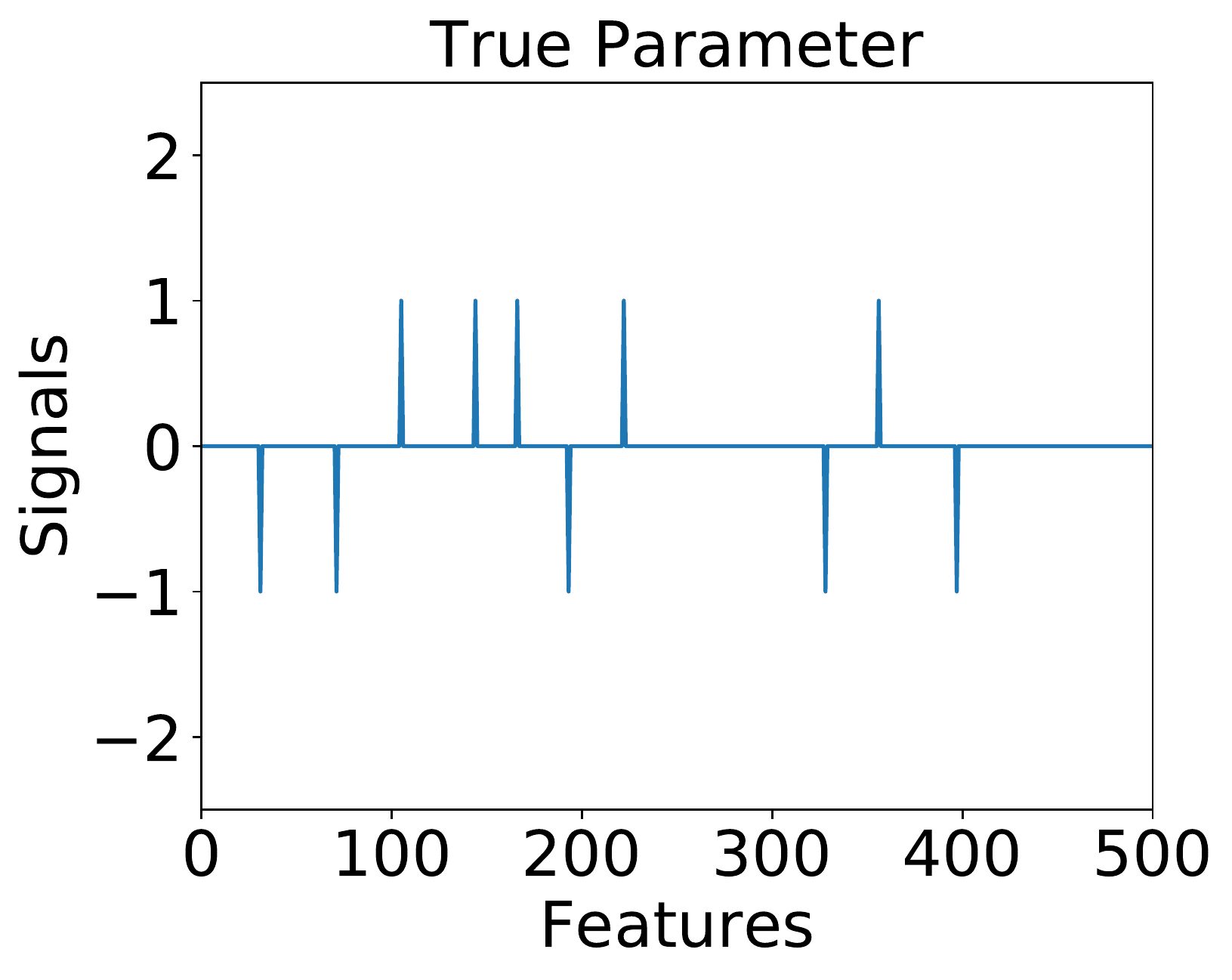}}
	\subfigure[\textsc{ProxGen}]{\includegraphics[width=0.32\linewidth]{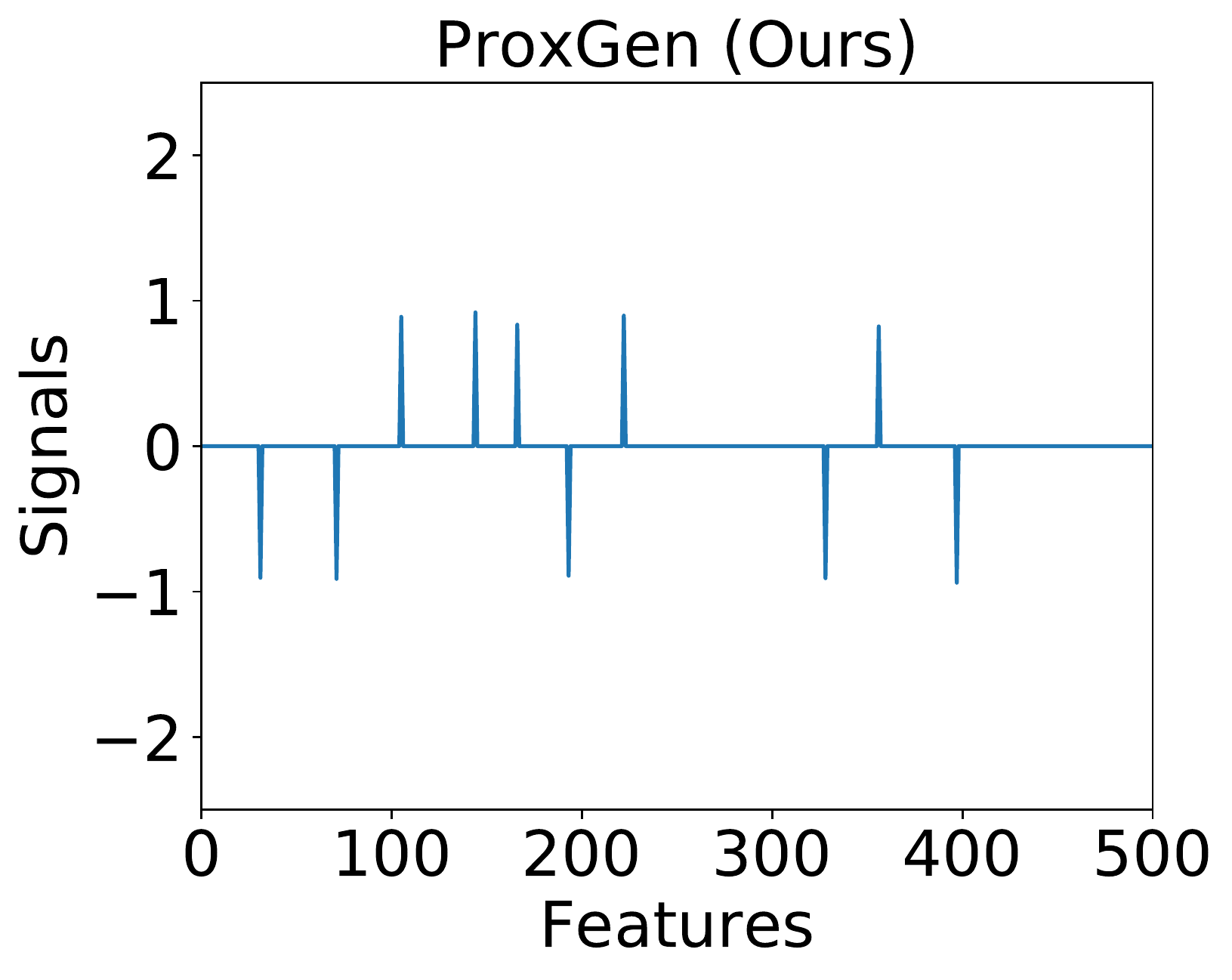}}
	\subfigure[\textsc{Prox-SGD}]{\includegraphics[width=0.32\linewidth]{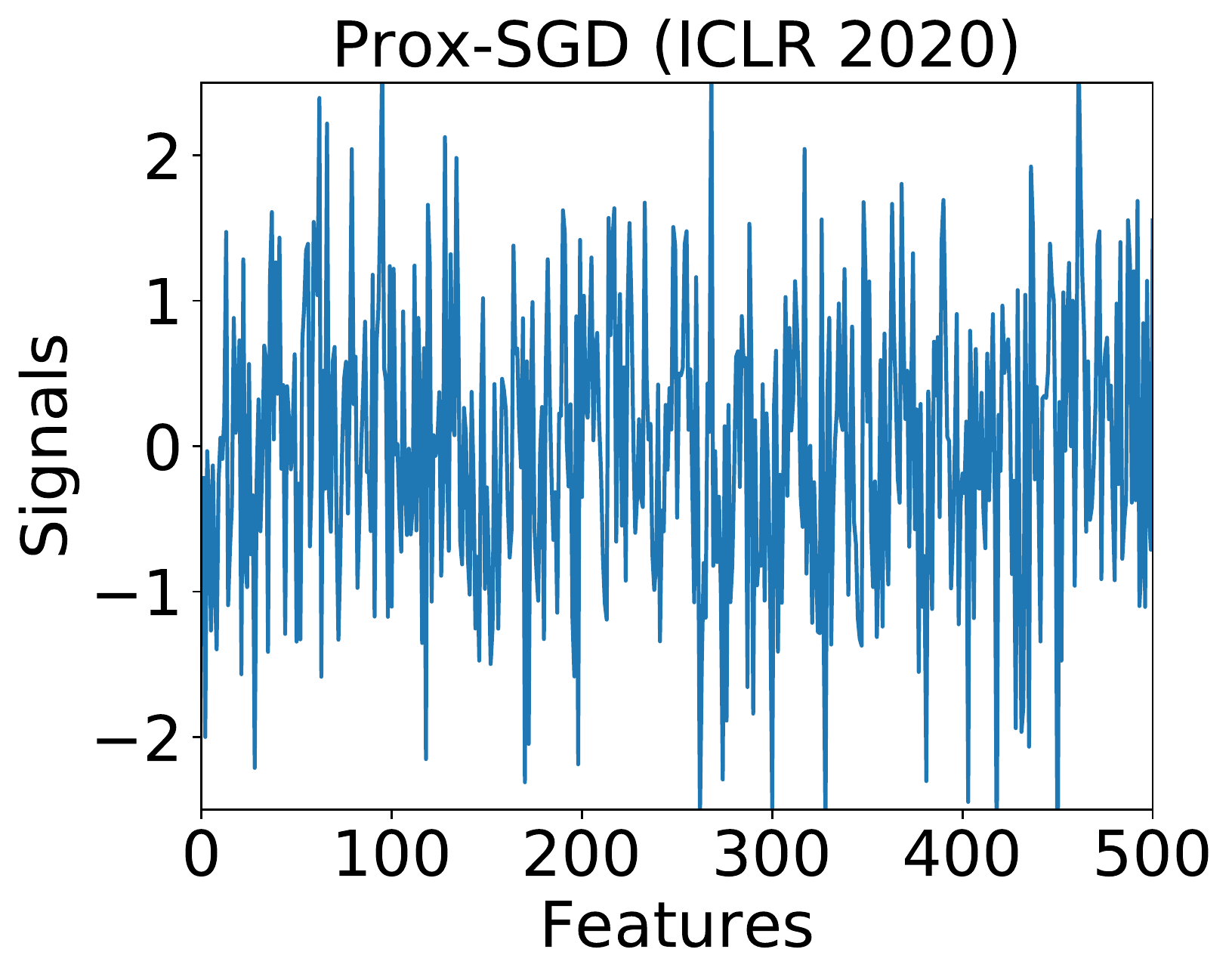}}
	\caption{Lasso simulations for support recovery with random initialization. \textsc{Prox-SGD} cannot recover the correct support.}
	\label{fig:lasso_random}
\end{figure}

\begin{figure}[t]
	\centering
	\subfigure[True parameter]{\includegraphics[width=0.32\linewidth]{figures/true_parameter.pdf}}
	\subfigure[\textsc{ProxGen}]{\includegraphics[width=0.32\linewidth]{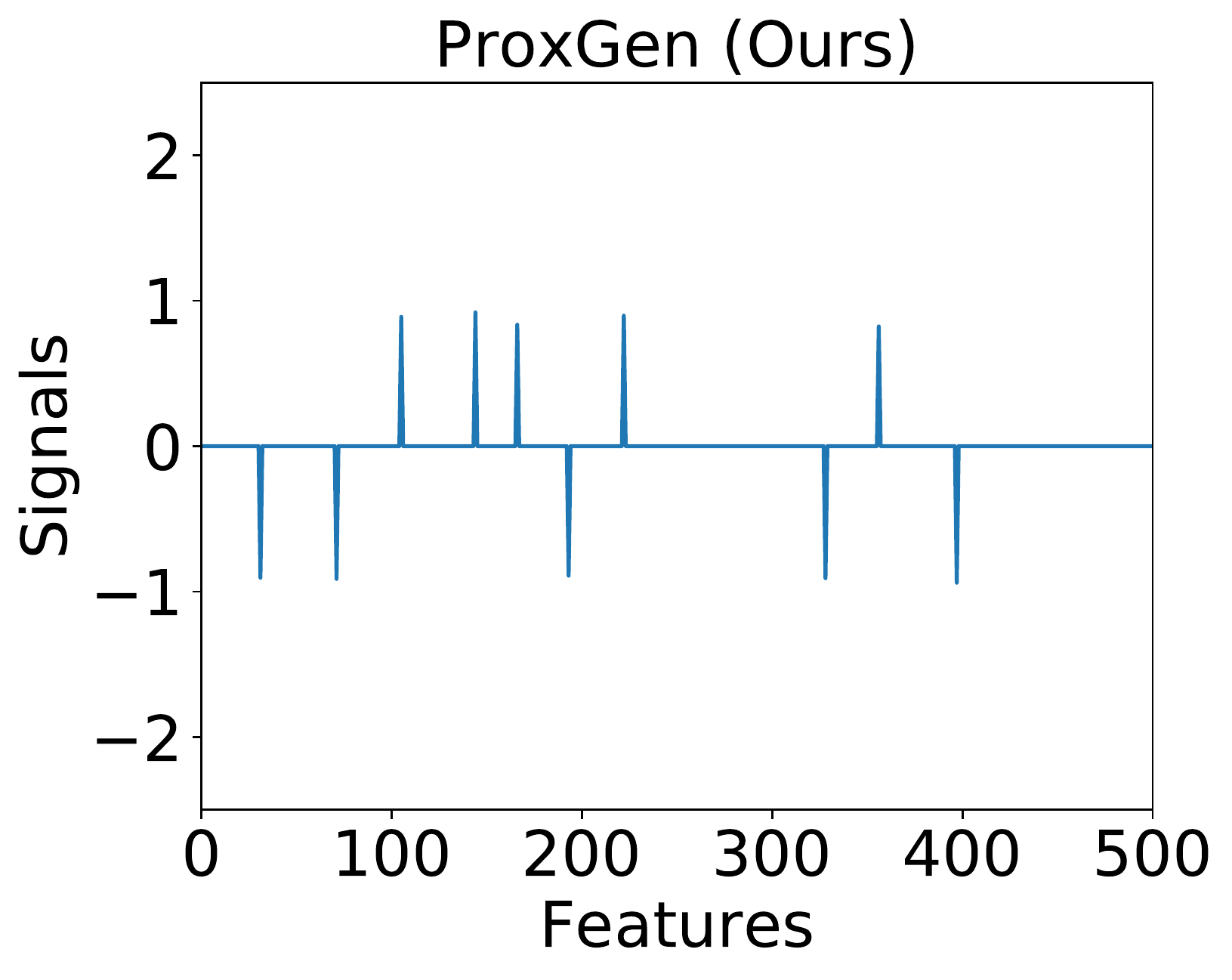}}
	\subfigure[\textsc{Prox-SGD}]{\includegraphics[width=0.32\linewidth]{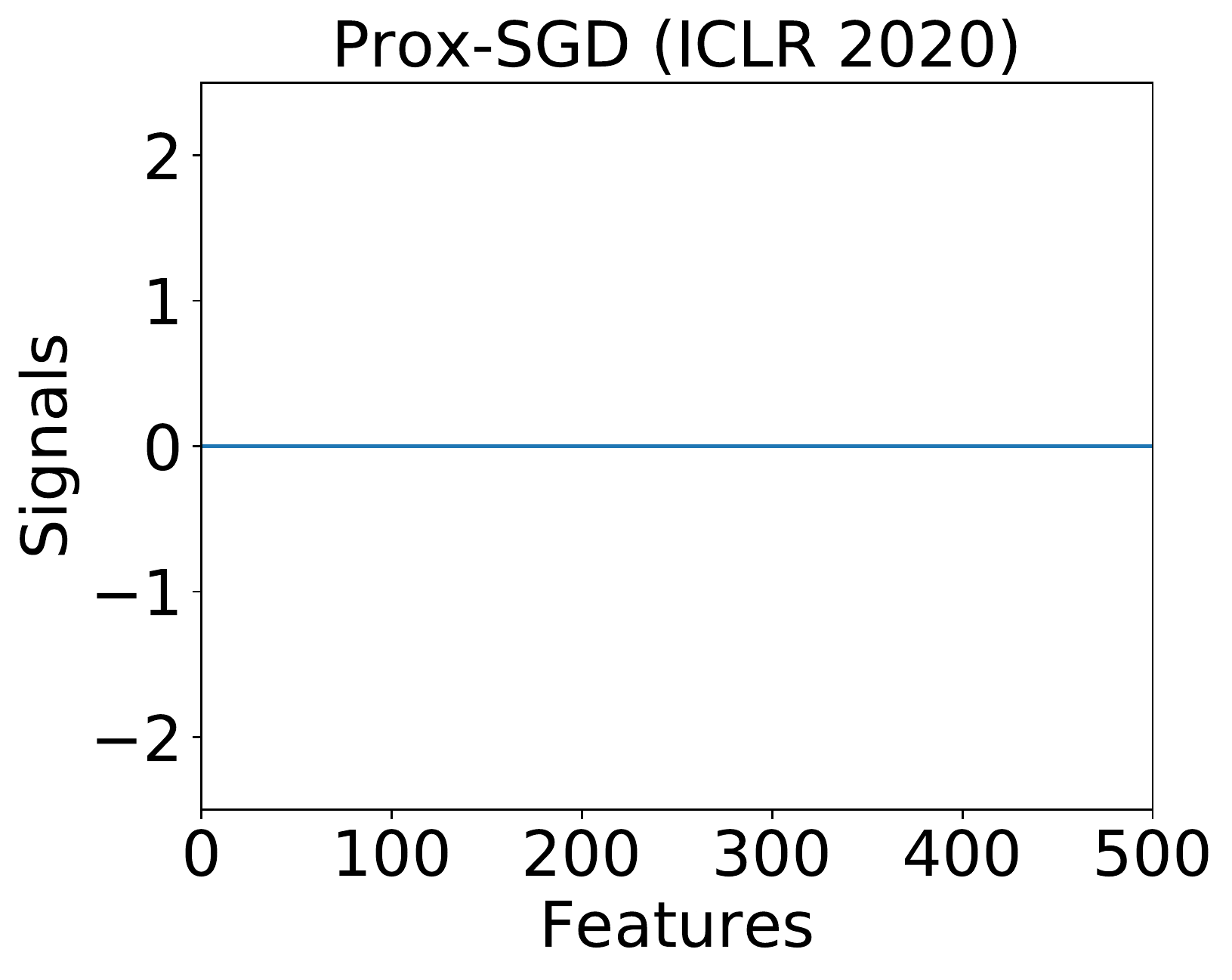}}
	\caption{Lasso simulations for support recovery with zero initialization. Note that \textsc{Prox-SGD} in this case overestimate the sparsity level at Equation \eqref{eqn:proxsgd_update}, so all the subsequently updated parameters $\theta_t$ also become zero.}
	\label{fig:lasso_zero}
\end{figure}

\begin{algorithm}[t]
	\caption{\textsc{ProxGenW}: A \textbf{Gen}eral Stochastic \textbf{Prox}imal Gradient Method with Weight Decay}
	\label{alg:proxgenw_alg}{
		\begin{algorithmic}[1]
			\State {\bfseries Input:} Stepsize $\alpha_t$, $\{\rho_t\}_{t=1}^{t=T} \in [0, 1)$, regularization parameter $\lambda$, small constant $0 < \delta <\!\!\!< 1$, and weight decay regularization parameter $\zeta$.
			\State {\bfseries Initialize:} $\theta_1 \in \mathbb{R}^{d}$, $m_0 = 0$, and $C_0 = 0$.
			\For{$t = 1, 2, \ldots, T$}
			\State{Draw a minibatch sample $\xi_t$ from $\mathbb{P}$}
			\Let{$g_t$}{$\nabla f(\theta_{t}; \xi_t)$} \Comment{Stochastic gradient at time $t$} 
			\Let{$m_t$}{$\rho_t m_{t-1} + (1 - \rho_t) g_t$} \Comment{First-order momentum estimate}\vspace{+0.1cm}
			\Let{$C_t$}{Preconditioner construction}
			\Let{$\bar{\theta}_t$}{$(1 - \alpha_t \zeta) \theta_t$} \Comment{Apply decoupled weight decay}
			\State{$\theta_{t+1} \in \argmin\limits_{\theta \in \Omega}\Big\{\langle m_t, \theta \rangle + \lambda \mathcal{R}(\theta) + \dfrac{1}{2\alpha_t}(\theta - \bar{\theta}_{t})^\mathsf{T} \big(C_t + \delta I) (\theta - \bar{\theta}_{t})\Big\}$}
			\EndFor
			\State {\bfseries Output:} $\theta_T$
	\end{algorithmic}}
\end{algorithm}

\section{Examples Satisfying Condition \ref{con:mineig}}

\begin{theorem}[Weyl]\label{thm:weyl}
	For any two $n \times n$ Hermitian matrices $A$ and $B$, assume that the eigenvalues of $A$ and $B$ are
	\begin{align*}
	\mu_1 \geq \cdots \geq \mu_n, \quad \text{ and } \quad \nu_1 \geq \cdots \geq \nu_n
	\end{align*}
	respectively. Let $\lambda_1 \geq \cdots \geq \lambda_n$ be the eigenvalues of the matrix $A + B$, then the following holds
	\begin{align*}
	\mu_j + \nu_k \leq \lambda_i \leq \mu_r + \nu_s
	\end{align*}
	for $j + k -n \geq i \geq r + s - 1$. Hence, we could derive
	\begin{align*}
	\lambda_1 \leq \mu_1 + \nu_1
	\end{align*}
\end{theorem}

We provide concrete examples and derivations satisfying Condition \ref{con:mineig} in Section \ref{sec:analysis}.

\paragraph{Vanilla \textsc{Sgd}.} The vanilla \textsc{Sgd} corresponds to $C_t = I$. We assume the constant stepsize $\alpha_t = \alpha$. Then, the condition \ref{con:mineig} can be computed as
\begin{align*}
\lambda_{\mathrm{min}}(\alpha_t (C_t + \delta I)^{-1}) = \lambda_{\mathrm{min}}(\alpha \frac{1}{\delta + 1}I) = \frac{\alpha}{\delta + 1}
\end{align*}
Therefore, we conclude that $\gamma = \frac{\alpha}{\delta + 1}$.

\paragraph{\textsc{AdaGrad}.} In \textsc{ProxGen} framework, \textsc{AdaGrad} corresponds to $C_t = \Big(\frac{1}{t}\sum\limits_{\tau=1}^{t} g_\tau g_\tau^\mathsf{T}\Big)^{1/2}$. Under the constant stepsizes $\alpha_t = \alpha $, we have
\begin{align*}
\lambda_{\mathrm{max}}(C_t) & = \frac{1}{\sqrt{t}}\lambda_{\mathrm{max}}\Big(\sum\limits_{\tau=1}^{t} g_\tau g_\tau^\mathsf{T}\Big)^{1/2} \\
& \leq \frac{1}{\sqrt{t}} \Big(\sum\limits_{\tau=1}^{t} \lambda_{\mathrm{max}}(g_\tau g_\tau^\mathsf{T})\Big)^{1/2} \\
& = \frac{1}{\sqrt{t}} \Big(\sum\limits_{\tau=1}^{t} \|g_\tau\|_2^2\Big)^{1/2} \\
& \leq G
\end{align*}
Hence, the Condition \ref{con:mineig} can be satisfied as
\begin{align*}
\lambda_{\mathrm{min}}(\alpha_t (C_t + \delta I)^{-1}) \geq \frac{\alpha}{G + \delta} \coloneqq \gamma
\end{align*}

\paragraph{\textsc{RMSprop} and \textsc{Adam}.} Exponential moving average (a.k.a. EMA) approaches correspond to $C_t = \big(\beta C_{t-1} + (1 - \beta) g_t g_t^\mathsf{T}\big)^{1/2}$ where $\beta \in [0, 1)$ and $g_t$ denotes the stochastic gradient at time $t$. The usual \textsc{RMSprop} and \textsc{Adam} use diagonal approximations for $g_t g_t^\mathsf{T}$, but here we consider more general form (i.e. including general full matrix gradient outer-product) as introduce in \cite{yun2019stochastic}. First, we derive the upper bound for maximum eigenvalue for the matrix $C_t$. The matrix $C_t$ can be expressed by
\begin{align*}
C_t & = \big(\beta C_{t-1} + (1 - \beta) g_t g_t^\mathsf{T}\big)^{1/2} \\
& = \big(\beta^2 C_{t-2} + \beta (1 - \beta) g_{t-1} g_{t-1}^\mathsf{T} + (1 - \beta) g_t g_t^\mathsf{T}\big)^{1/2} \\
& = \cdots \\
& = \Big((1 - \beta) \sum\limits_{i=1}^{t} \beta^{t-i} g_i g_i^\mathsf{T}\Big)^{1/2}
\end{align*}
We can derive the upper bound for maximum eigenvalue of $C_t$ using Weyl's theorem (Theorem \ref{thm:weyl}) by
\begin{align*}
\lambda_{\mathrm{max}}(C_t) & = \lambda_{\mathrm{max}}\Big((1 - \beta) \sum\limits_{i=1}^{t} \beta^{t-i} g_i g_i^\mathsf{T} \Big)^{1/2} \\
& \leq \Big((1 - \beta)\sum\limits_{i=1}^{t} \beta^{t-i} \lambda_{\mathrm{max}}(g_i g_i^\mathsf{T})\Big)^{1/2} \\
& \leq \Big((1 - \beta) G^2 \sum\limits_{i=1}^{t} \beta^{t-i}\Big)^{1/2} \\
& \leq G (1 - \beta^t)^{1/2} \leq G
\end{align*}
Hence, we have $\lambda_{\mathrm{max}}(C_t + \delta I) \leq G + \delta$. Also, we have
\begin{align*}
\lambda_{\mathrm{max}}\big(C_t + \delta I\big) = \frac{1}{\lambda_{\mathrm{min}}\big((C_t + \delta I)^{-1}\big)} \leq \frac{1}{G + \delta}
\end{align*}
Therefore, the condition \ref{con:mineig} under the constant stepsize $\alpha_t = \alpha$ can be derived as
\begin{align*}
\lambda_{\mathrm{min}}\big(\alpha_t (C_t + \delta I)^{-1}\big) \geq \frac{\alpha}{G + \delta}
\end{align*}
which yields $\gamma = \frac{\alpha}{G + \delta}$.

\paragraph{Natural Gradient Descent.} In this case, we derive the condition  \ref{con:mineig} for the Fisher information matrix when the loss function is defined as a negative log-likelihood, i.e., $f = \log p(x|\theta)$. The natural gradient descent aims at considering general geometry (not limited to Euclidean geometry), but we restrict our focus on the distribution space where the Fisher information is employed for preconditioner matrix $C_t$. The Fisher information matrix is defined as
\begin{align*}
F = \mathbb{E}_{Q(x) P(y|x,\theta)}\Big[\frac{\partial f(x|\theta)}{\partial \theta} \frac{\partial f(x|\theta)}{\partial \theta}^\mathsf{T} \Big]
\end{align*}
where $Q(x)$ is data distribution and $P(y|x,\theta)$ denotes the model's predictive distribution (ex. neural networks). However, in general, we do not have access to true data distribution, so we instead take an expectation with respect to empirical (training) data distribution $\widehat Q(x)$. This trick is also employed for K-FAC approximations to the Fisher \cite{martens2015optimizing}. Let the training samples be $\mathcal{S} = \{x_1, \cdots, x_n\}$ with sample size $n$. Then, the empirical Fisher could be computed as
\begin{align*}
\widehat F & = \mathbb{E}_{\widehat Q(x) P(y|x,\theta)}\Big[\frac{\partial f(x|\theta)}{\partial \theta} \frac{\partial f(x|\theta)}{\partial \theta}^\mathsf{T}\Big] \\ 
& = \frac{1}{n} \sum\limits_{i=1}^{n} \frac{\partial f(x_i|\theta)}{\partial \theta} \frac{\partial f(x_i|\theta)}{\partial \theta}^\mathsf{T} 
\end{align*}
Now, we bound the maximum eigenvalue of $\widehat F$ as
\begin{align*}
\lambda_{\mathrm{max}}(\widehat F) & = \frac{1}{n} \sum\limits_{i=1}^{t} \lambda_{\mathrm{max}}\big(\frac{\partial f(x_i|\theta)}{\partial \theta} \frac{\partial f(x_i|\theta)}{\partial \theta}^\mathsf{T}\big) \\
& \leq \frac{1}{n} \sum\limits_{i=1}^{t} G^2 \\
& = G^2
\end{align*}
by our Condition \ref{con:mild}. Hence, the Condition \ref{con:mineig} can be derived as
\begin{align*}
\lambda_{\mathrm{min}}\big(\alpha_t (\widehat F + \delta I)^{-1}\big) \geq \frac{\alpha}{G^2 + \delta}
\end{align*}
under the constant stepsize $\alpha_t = \alpha$.

\section{Proofs of Theorem 1}
\begin{lemma}\label{lem:momentum_bound}
	The first-order momentum $m_t$ in Algorithm \ref{alg:proxgen_alg} satisfies
	\begin{align*}
	\|m_t\|_2 \leq G
	\end{align*}
\end{lemma}
\begin{proof}
	We use mathematical induction. For $t = 1$, the momentum is computed as $m_1 = \rho_1 m_0 + (1 - \rho_1) g_t = (1 - \rho_0) g_1$. Therefore, we have $\|m_t\|_2 = \|(1 - \rho_0) g_1\| \leq (1 - \rho_0) G \leq G$.
	
	Now, we assume that $\|m_{t-1}\|_2 \leq G$ holds. The momentum at time $t$ is constructed by $m_t = (1 - \rho_t) m_{t-1} + \rho_t g_t$. Then, we have
	\begin{align*}
	\|m_t\|_2 & = \|(1 - \rho_t) m_{t-1} + \rho_t g_t\|_2 \\
	& \leq (1 - \rho_t)\|m_{t-1}\|_2 + \rho_t \|g_t\|_2 \\
	& \leq (1 - \rho_t) G + \rho_t G = G 
	\end{align*}
	where the first inequality comes from the triangle inequality and the second one is derived from the induction hypothesis.
\end{proof}
We deal with the following update rule in Algorithm  \ref{alg:proxgen_alg} as
\begin{align}\label{eqn:update_def}
\theta_{t+1} \in \argmin\limits_{\theta \in \Omega} \Big\{\big\langle (1 - \rho_t) g_t + \rho_t m_{t-1}, \theta \big\rangle + \mathcal{R}(\theta) + \frac{1}{2\alpha_t} (\theta - \theta_t)^\mathsf{T} (C_t + \delta I) (\theta - \theta_t) \Big\}
\end{align}
By the optimality condition, we have
\begin{align*}
0 \in (1 - \rho_t) g_t + \rho_t m_{t-1} + \widehat{\partial}\mathcal{R}(\theta_{t+1}) + \frac{1}{\alpha_t} (C_t + \delta I) (\theta_{t+1} - \theta_t)
\end{align*}
which means that
\begin{align*}
-(1 - \rho_t) g_t - \rho_t m_{t-1} - \frac{1}{\alpha_t} (C_t + \delta I)  (\theta_{t+1} - \theta_t) \in \widehat{\partial} \mathcal{R}(\theta_{t+1}) 
\end{align*}
By adding the gradient $\nabla f(\theta_{t+1})$ on both sides, we have
\begin{align*}
\nabla f(\theta_{t+1}) - (1 - \rho_t) g_t - \rho_t m_{t-1} - \frac{1}{\alpha_t} (C_t + \delta I)(\theta_{t+1} - \theta_t) \in \nabla f(\theta_{t+1}) + \widehat{\partial}\mathcal{R}(\theta_{t+1}) = \widehat{\partial} F(\theta_{t+1})
\end{align*}
By the definition of $\theta_{t+1}$ in \eqref{eqn:update_def}, we obtain
\begin{align*}
& \big\langle (1 - \rho_t) g_t + \rho_t m_{t-1}, \theta_{t+1} \big\rangle + \mathcal{R}(\theta_{t+1}) + \frac{1}{2\alpha_t}(\theta_{t+1} - \theta_t)^\mathsf{T} (C_t + \delta I) (\theta_{t+1} - \theta_t) \\
\leq~ & \big\langle (1 - \rho_t)g_t + \rho_t m_{t-1}, \theta_t \big\rangle + \mathcal{R}(\theta_t)
\end{align*}
which in result
\begin{align*}
\big\langle (1 - \rho_t)g_t + \rho_t m_{t-1}, \theta_{t+1} - \theta_t \big\rangle + \mathcal{R}(\theta_{t+1}) + \frac{1}{2\alpha_t} (\theta_{t+1} - \theta_t)^\mathsf{T} (C_t + \delta I) (\theta_{t+1} - \theta_t) \leq \mathcal{R}(\theta_t) 
\end{align*}
Since the function $f$ is $L$-smooth by Condition \ref{con:smooth}, we have
\begin{align*}
f(\theta_{t+1}) \leq f(\theta_t) + \langle \nabla f(\theta_t), \theta_{t+1} - \theta_t \rangle + \frac{L}{2} \|\theta_{t+1} - \theta_t \|_2^2
\end{align*}
Adding previous two inequalities yields
\begin{align}\label{eqn:core}
& \big\langle (1 - \rho_t) g_t - \nabla f(\theta_t) + \rho_t m_{t-1}, \theta_{t+1} - \theta_t \big\rangle + (\theta_{t+1} - \theta_t)^\mathsf{T} \Bigg(\frac{1}{2\alpha_t} (C_t + \delta I) - \frac{L}{2} I\Bigg) (\theta_{t+1} - \theta_t) \nonumber \\
\leq~& F(\theta_t) - F(\theta_{t+1})
\end{align}
Then, we have
\begin{align*}
& \|\theta_{t+1} - \theta_t\|_{\frac{1}{2\alpha_t}(C_t + \delta I) - \frac{L}{2}I}^2 \\
\overset{\textcircled{1}}{\leq}~& F(\theta_t) - F(\theta_{t+1}) - \big\langle (1 - \rho_t)g_t - \nabla f(\theta_t), \theta_{t+1} - \theta_t \big\rangle - \big\langle \rho_t m_{t-1}, \theta_{t+1} - \theta_t \big\rangle \\
=~& F(\theta_t) - F(\theta_{t+1}) - \big\langle g_t - \nabla f(\theta_t), \theta_{t+1} - \theta_t \big\rangle + \langle \rho_t g_t, \theta_{t+1} - \theta_t \rangle - \langle \rho_t m_{t-1}, \theta_{t+1} - \theta_t \rangle \\
\overset{\textcircled{2}}{\leq}~& F(\theta_t) - F(\theta_{t+1}) + \frac{1}{2L} \|g_t - \nabla f(\theta_t)\|_2^2 + \frac{L}{2} \|\theta_{t+1} - \theta_t\|_2^2 + \frac{\rho_t^2}{2L}\|g_t\|_2^2 + \frac{L}{2} \|\theta_{t+1} - \theta_t\|_2^2 \\
& \quad + \|\rho_t m_{t-1}\|_2 \|\theta_{t+1} - \theta_t\|_2 \\ 
\overset{\textcircled{3}}{\leq}~& F(\theta_t) - F(\theta_{t+1}) + \rho_0 \mu^{t-1} DG + 
\frac{\rho_0^2 \mu^{2(t-1)}G^2}{2L} + L \|\theta_{t+1} - \theta_t\|_2^2   + \frac{1}{2L} \|g_t - \nabla f(\theta_t)\|_2^2
\end{align*}
The derivations in inequalities (1-3) as follows:
\begin{enumerate}[label=\large\protect\textcircled{\small\arabic*}]
	\item We rearrange the inequality \eqref{eqn:core}.
	\item We use the fact that $\langle a, b \rangle \leq \frac{1}{2}\|a\|_2^2 + \frac{1}{2}\|b\|_2^2$ and $\langle a, b \rangle \leq \|a\|_2 \|b\|_2$. With this, we use modified version such as $\langle a, b \rangle = \langle ca, \frac{1}{c}b \rangle \leq c^2 \|a\|_2^2 + \frac{1}{c^2} \|b\|_2^2$ for any positive constant $c$. 
	\item We apply our Lemma \ref{lem:momentum_bound} and Condition \ref{con:mild}.
\end{enumerate}
By rearranging the above inequality, we require the following quantity be positive-semidefinite.
\begin{align*}
\frac{1}{2\alpha_t} (C_t + \delta I) - \frac{3}{2}L I \succeq 0
\end{align*}
Note that in this inequality we can see that
\begin{align*}
\frac{1}{2\alpha_t} (C_t + \delta I) - \frac{3}{2} LI \succeq \frac{1}{2\alpha_0} \delta I - \frac{3}{2} LI
\end{align*}
since $C_t$ is positive (semi)definite and $\alpha_t$ is \emph{non-increasing}. Therefore, from this we can derive the stepsize condition in our Theorem \ref{thm:general_convergence} as
\begin{align*}
\alpha_0 \leq \frac{\delta}{3L}
\end{align*}
Therefore, we have
\begin{align*}
\sum\limits_{t=0}^{T-1} \|\theta_{t+1} - \theta_t\|_{\frac{1}{2\alpha_t} (C_t + \delta I) - \frac{3}{2}L I}^2 & \leq \underbrace{F(\theta_0) - F(\theta^{*})}_{\Delta} + \underbrace{\frac{\rho_0DG}{1 - \mu} + \frac{\rho_0^2 G^2}{2L(1 - \mu^2)}}_{C_1} + \frac{1}{2L} \sum\limits_{t=0}^{T-1} \|g_t - \nabla f(\theta_t)\|_2^2 \\
& \leq \Delta + C_1 + \frac{1}{2L} \sum\limits_{t=0}^{T-1} \|g_t - \nabla f(\theta_t)\|_2^2 
\end{align*}
Furthermore, we also have by stepsize condition
\begin{align*}
\Big(\frac{\delta}{2\alpha_0} - \frac{3}{2}L\Big) \sum\limits_{t=0}^{T-1} \|\theta_{t+1} - \theta_t\|_2^2 \leq \sum\limits_{t=0}^{T-1} \|\theta_{t+1} - \theta_t\|_{\frac{1}{2\alpha_t} (C_t + \delta I) - \frac{3}{2}L I}^2 \leq \Delta + C_1 + \frac{1}{2L} \sum\limits_{t=0}^{T-1} \|g_t - \nabla f(\theta_t)\|_2^2 
\end{align*}
since $\delta I \preceq C_t + \delta I$. From above inequality, we obtain
\begin{align}\label{eqn:distance_bound}
\sum\limits_{t=0}^{T-1} \|\theta_{t+1} - \theta_t\|_2^2 \leq H_1 + H_2 \sum\limits_{t=0}^{T-1}\|g_t - \nabla f(\theta_t)\|_2^2
\end{align}
where the constants $H_1$ and $H_2$ are defined as
\begin{align*}
H_1 & = \Delta \Bigg/ \Big(\frac{\delta}{2\alpha_0} - \frac{3}{2}L\Big) + C_1 \Bigg/ \Big(\frac{\delta}{2\alpha_0} - \frac{3}{2}L\Big) \\
H_2 & = \frac{1}{2L(\frac{\delta}{2\alpha_0} - \frac{3}{2}L)}
\end{align*}
Our goal is to bound the distance between the zero vector and subdifferential set of $F$, so we have
\begin{align*}
& \mathrm{dist}(\bm{0}, \widehat{\partial}F(\theta_{t+1}))^2 \\
=~& \Big\|(1 - \rho_t)g_t - \nabla f(\theta_{t+1}) + \rho_t m_{t-1} + \frac{1}{\alpha_t} (C_t + \delta I)(\theta_{t+1} - \theta_t)\Big\|_2^2 \\
=~& \Big\|(1 - \rho_t)g_t - \nabla f(\theta_{t+1}) + \rho_t m_{t-1} + (\theta_{t+1} - \theta_t) + \frac{1}{\alpha_t} (C_t + \delta I)(\theta_{t+1} - \theta_t) - (\theta_{t+1} - \theta_t)\Big\|_2^2 \\
\leq~& 3\Big\|(1 - \rho_t)g_t - \nabla f(\theta_{t+1}) + \rho_t m_{t-1} + (\theta_{t+1} - \theta_t)\Big\|_2^2 \\
\quad+~& 3\Big\|\frac{1}{\alpha_t} (C_t + \delta I)(\theta_{t+1} - \theta_t)\Big\|_2^2 + 3\Big\|(\theta_{t+1} - \theta_t)\Big\|_2^2 \\
\leq~& 3\underbrace{\Big\|(1 - \rho_t)g_t - \nabla f(\theta_{t+1}) + \rho_t m_{t-1} + (\theta_{t+1} - \theta_t)\Big\|_2^2}_{T_1} + 3\Big(\frac{1}{\gamma^2} + 1\Big)\|\theta_{t+1} - \theta_t\|_2^2
\end{align*}
Here, we assume that
\begin{align*}
\lambda_{\mathrm{max}}\big(\frac{1}{\alpha_t} (C_t + \delta I)\big) \leq \frac{1}{\gamma} 
\end{align*}
which yields our Condition \ref{con:mineig}
\begin{align*}
\lambda_{\mathrm{min}}\big(\alpha_t (C_t + \delta I)^{-1}\big) \geq \gamma
\end{align*}
From \eqref{eqn:core}, we have
\begin{align*}
\big\langle (1 - \rho_t)g_t - \nabla f(\theta_t) + \rho_t m_{t-1}, \theta_{t+1} - \theta_t \big\rangle + \big\|\theta_{t+1} - \theta_t\|_{\frac{1}{2\alpha_t}(C_t + \delta I) - \frac{L}{2} I}^2 \leq F(\theta_t) - F(\theta_{t+1})
\end{align*}
which can be re-written as
\begin{align*}
& \Big\langle (1 - \rho_t) g_t - \nabla f(\theta_{t+1}) + \rho_t m_{t-1}, \theta_{t+1} - \theta_t \Big\rangle \\
\leq~& F(\theta_t) - F(\theta_{t+1}) - \big\langle \nabla f(\theta_{t+1}) - \nabla f(\theta_t), \theta_{t+1} - \theta_t \big\rangle - \big\|\theta_{t+1} - \theta_t\|_{\frac{1}{2\alpha_t}(C_t + \delta I) - \frac{L}{2} I}^2 \\
\leq~& F(\theta_t) - F(\theta_{t+1}) - \big\langle \nabla f(\theta_{t+1}) - \nabla f(\theta_t), \theta_{t+1} - \theta_t \big\rangle + \Big(\frac{\delta}{2\alpha_0} - \frac{L}{2}\Big)\|\theta_{t+1} - \theta_t\|_2^2
\end{align*}
since we have the condition $\frac{\delta}{2\alpha_0} \geq \frac{3}{2}L$. Therefore, we obtain
\begin{align*}
T_1 & = \|(1 - \rho_t)g_t - \nabla f(\theta_{t+1}) + \rho_t m_{t-1}\|_2^2 + \|\theta_{t+1} - \theta_t\|_2^2 \\
& \quad + 2\Big\langle (1 - \rho_t) g_t - \nabla f(\theta_{t+1}) + \rho_t m_{t-1}, \theta_{t+1} - \theta_t \Big\rangle \\
& \leq \|(1 - \rho_t)g_t - \nabla f(\theta_t) + \nabla f(\theta_t) - \nabla f(\theta_{t+1}) + \rho_t m_{t-1}\|_2^2 + \|\theta_{t+1} - \theta_t\|_2^2 \\
& \quad + F(\theta_t) - F(\theta_{t+1}) - \big\langle \nabla f(\theta_{t+1}) - \nabla f(\theta_t), \theta_{t+1} - \theta_t \big\rangle + \Big(\frac{\delta}{2\alpha_0} - \frac{L}{2}\Big)\|\theta_{t+1} - \theta_t\|^2 \\
& \leq 4\|g_t - \nabla f(\theta_t)\|_2^2 + 4L^2\|\theta_{t+1} - \theta_t\|_2^2 + 4\|\rho_t m_{t-1}\|_2^2 + 4\|\rho_t g_t \|_2^2 + \|\theta_{t+1} - \theta_t\|_2^2 \\
& \quad + F(\theta_t) - F(\theta_{t+1}) + L\|\theta_{t+1} - \theta_t\|_2^2 + \Big(\frac{\delta}{2\alpha_0} - \frac{L}{2}\Big)\big\|\theta_{t+1} - \theta_t\|_2^2 \\
& \leq F(\theta_t) - F(\theta_{t+1}) + 4\rho_0^2 \mu^{2(t-1)} G^2 + 4\rho_0^2 \mu^{2(t-1)}G^2 \\
& \quad + \Big(\frac{\delta}{2\alpha_0} + \frac{L}{2} + 1 + 4L^2\Big) \|\theta_{t+1} - \theta_t\|_2^2 + 4 \|g_t - \nabla f(\theta_t)\|_2^2 
\end{align*}
Therefore, we have the distance as
\begin{align*}
& \mathrm{dist}\big(\bm{0}, \widehat{\partial}F(\theta_{t+1})\big)^2 \\
\leq~& 3\Bigg(F(\theta_t) - F(\theta_{t+1}) + 8\rho_0^2 \mu^{2(t-1)} G^2 + \Big(\underbrace{\frac{\delta}{2\alpha_0} + \frac{L}{2} + 2 + 4L^2 + \frac{1}{\gamma^2}}_{C_2}\Big)\|\theta_{t+1} - \theta_t\|_2^2 + 4\|g_t - \nabla f(\theta_t)\|_2^2\Bigg) 
\end{align*}
Therefore, we have
\begin{align*}
\mathbb{E}[\mathrm{dist}\big(\bm{0}, \widehat{\partial}F(\theta_a)\big)^2] & \leq \frac{1}{T} \sum\limits_{t=0}^{T-1} \mathbb{E}\Big[\big\|(1 - \rho_t)g_t - \nabla f(\theta_{t+1}) + \rho_t m_{t-1} + \frac{1}{\alpha_t} (C_t + \delta I)(\theta_{t+1} - \theta_t)\big\|_2^2\Big] \\
& \leq \frac{3}{T} \Big(\Delta + \frac{8\rho_0^2 G^2}{1 - \mu^2} + 4\sum\limits_{t=0}^{T-1} \|g_t - \nabla f(\theta_t)\|_2^2 + C_2 \sum\limits_{t=0}^{T-1} \|\theta_{t+1} - \theta_t\|_2^2\Big) \\ 
& \leq \frac{3}{T}\Big(\Delta + \frac{8\rho_0^2 G^2}{1 - \mu^2} + 4\sum\limits_{t=0}^{T-1} \|g_t - \nabla f(\theta_t)\|_2^2 + C_2 (H_1 + H_2 \sum\limits_{t=0}^{T-1} \|g_t - \nabla f(\theta_t)\|_2^2\Big) \\
& \leq \frac{Q_1}{T} \sum\limits_{t=0}^{T-1} \mathbb{E}\big[\|g_t - \nabla f(\theta_t)\|_2^2\big] + \frac{Q_2\Delta}{T} + \frac{Q_3}{T}
\end{align*}
where
\begin{align*}
Q_1 = 4 + C_2 H_2, \quad Q_2 = 3 + \frac{3C_2}{\frac{\delta}{2\alpha_0} - \frac{3}{2}L}, \quad Q_3 = \frac{24\rho_0^2 G^2}{1 - \mu^2} + \frac{3 C_1 C_2}{\frac{\delta}{2\alpha_0} - \frac{3}{2}L} 
\end{align*}
Note that the constants $Q_1$, $Q_2$, and $Q_3$ depend on $\{\alpha_0, \delta, L, D, G, \rho_0, \mu, \gamma\}$, but not on $T$.
The third inequality comes from \eqref{eqn:distance_bound}. If we assume the stochastic gradient $g_t$ is evaluated on the minibatch $\mathcal{S}_t$ with $|\mathcal{S}_t| = b_t$, then we can obtain using Condition \ref{con:var}
\begin{align*}
	\|g_t - \nabla f(\theta_t)\|_2^2 & = \mathbb{E}\Big[\Big\|\Big(\frac{1}{b_t} \sum\limits_{i=1}^{b_t} \nabla f(\theta_t; \xi_{i_t})\Big) - \nabla f(\theta_t)]\Big\|_2^2\Big] \\
	& = \frac{1}{b_t^2} \mathbb{E}\Big[\Big\|\sum\limits_{i=1}^{b_t} \big\{\nabla f(\theta_t; \xi_{i_t}) - \nabla f(\theta_t)\big\}\Big\|_2^2\Big] \\
	& \leq \frac{1}{b_t^2} \sum\limits_{i=1}^{b_t} \mathbb{E}\big[\|\nabla f(\theta_t; \xi_{i_t}) - \nabla f(\theta_t)\|_2^2\big]\\ 
	& \leq \frac{1}{b_t} \sigma^2 
\end{align*}
where $i_t$ represents the random variable for each datapoint in minibatch samples $\mathcal{S}_t$.
Finally, we arrive at our Theorem \ref{thm:general_convergence} as
\begin{align*}
	\mathbb{E}_R[\mathrm{dist}\big(\bm{0}, \widehat{\partial}F(\theta_R)\big)^2] \leq \frac{Q_1 \sigma^2}{T} \sum\limits_{t=0}^{T-1} \frac{1}{b_t} + \frac{Q_2\Delta}{T} + \frac{Q_3}{T}
\end{align*}

\end{document}